\documentclass{article}

\usepackage{arxiv}
\usepackage{natbib}
\usepackage{enumitem}

\usepackage[utf8]{inputenc} 
\usepackage[T1]{fontenc}    
\usepackage{hyperref}       
\usepackage{url}            
\usepackage{booktabs}       
\usepackage{amsfonts}       
\usepackage{nicefrac}       
\usepackage{microtype}      
\usepackage{xcolor}         

\usepackage{microtype}
\usepackage{graphicx}
\usepackage{subfig}

\usepackage{amsmath}
\usepackage{amssymb}
\usepackage{mathtools}
\usepackage{amsthm}

\usepackage[capitalize,noabbrev]{cleveref}
\usepackage{algpseudocode}
\usepackage{algorithm}
\usepackage{wrapfig}
\usepackage{cleveref}
\usepackage{minitoc}
\usepackage{subcaption}
\usepackage{longtable}

\theoremstyle{plain}
\newtheorem{theorem}{Theorem}[section]
\newtheorem{proposition}[theorem]{Proposition}
\newtheorem{lemma}[theorem]{Lemma}
\newtheorem{corollary}[theorem]{Corollary}
\theoremstyle{definition}

\theoremstyle{remark}

\allowdisplaybreaks
\newcommand{\R}{\mathbb{R}}
\newcommand{\E}{\mathbb{E}}
\newcommand{\Ptwo}{\mathcal{P}_2}
\newcommand{\N}{\mathcal{N}}

\newcommand{\Cov}{\operatorname{Cov}}
\newcommand{\Tr}{\operatorname{Tr}}

\newcommand{\diag}{\operatorname{diag}}
\newcommand{\argmin}{\operatorname*{arg\,min}}

\newcommand{\dd}{\,\mathrm{d}}
\newcommand{\Id}{I_d}

\newcommand{\G}{\mathcal{G}}

\newcommand{\SBB}{\mathsf{SBB}}
\newcommand{\sbbcost}{\mathsf{SBB}_{\beta,\varepsilon,T}}
\newcommand{\eps}{\varepsilon}

\newcommand{\Bdec}{\mathcal{B}_{\mathrm{dec}}}
\newcommand{\Benc}{\mathcal{B}_{\mathrm{enc}}}
\newcommand{\Bsym}{\mathcal{B}_{\mathrm{sym}}}
\newcommand{\rhodec}{\rho_{\mathrm{dec}}}
\newcommand{\rhoenc}{\rho_{\mathrm{enc}}}
\newcommand{\rhosym}{\rho_{\mathrm{sym}}}

\newcommand{\norm}[1]{\lVert #1\rVert}

\newcommand{\qedbox}{\begin{center}$\square$\end{center}}

\hypersetup{pdftitle={Schrodinger--Bass Barycenters},pdfauthor={Anonymous Authors}}
\makeatletter
\providecommand{\theHALG@line}{}
\renewcommand{\theHALG@line}{\thealgorithm.\arabic{ALG@line}}
\newenvironment{breakablealgorithm}
{\begin{center}\refstepcounter{algorithm}\hrule height 0.8pt depth 0pt\kern 2pt
\renewcommand{\caption}[2][\relax]{\raggedright\textbf{\ALG@name~\thealgorithm.} ##2\par\kern 2pt\hrule\kern 2pt}}
{\kern 2pt\hrule\end{center}}
\makeatother

\title{Schr\"odinger-Bass Bridge Barycenter Models for Measure-Valued Data}

\author{
 Ngoc-Hai Nguyen \\
  Department of Electrical and Computer Engineering\\
  Tufts University\\
  \texttt{Ngoc-Hai.Nguyen@tufts.edu} \\
   \And
 Thuan Nguyen \\
    Department of Engineering, Engineering Technology \\
    East Tennessee State University \\
    \texttt{nguyent11@etsu.edu} \\
  \And
 Prakash Ishwar \\
    Department of Electrical and Computer Engineering \\
    Boston University\\
    \texttt{pi@bu.edu}
\And
 Shuchin Aeron \\
    Department of Electrical and Computer Engineering \\
    Tufts University\\
    NSF AI Institute for Artificial Intelligence and Fundamental Interactions\\
    \texttt{shuchin@ece.tufts.edu}
 \\
}

\begin{document}
\maketitle
\begin{abstract}
    Barycenter models, based on stochastic and deterministic optimal transport bridges between probability measures, have proven to be useful for supervised and unsupervised learning for (probability) measure-valued data. In this paper, we utilize the directional asymmetry and dual drift-diffusion control arising in the recently proposed Schr\"odinger-Bass Bridge (SBB) between two measures, and propose \emph{SBB-encoder}, \emph{SBB-decoder}, \emph{SBB-symmetric}, and \emph{SBB-Markovian} barycenter models. For each of these novel models, our first theoretical result establishes the existence of the SBB-barycenter under mild assumptions. Our second theoretical result characterizes the decoder SBB-barycenter of Gaussians where we show that for Gaussian input measures, the decoder formulation admits a Gaussian minimizer, reducing barycenter computation to a finite-dimensional optimization and yielding a compact algorithm. For general input measures, we derive a hybrid Wasserstein gradient particle descent scheme for all four formulations using the SBB dual representation. Finally, we utilize these algorithms for solving distributional-inverse/interpolation problems arising in the context of spatial images and single-cell RNA sequencing data, where we show that the SBB-decoder and Markovian barycenter allow us to interpolate between successive time points. Results show that our method outperforms existing barycenter models for measure interpolation and inference. 
\end{abstract}

\section{INTRODUCTION}
\label{sec:introduction}

Barycenter models, that use metrics and divergences between probability measures, provide a means for averaging and analyzing (probability) measure-valued data, where each observation may itself be a distribution rather than a single vector, in both supervised and unsupervised settings. Among the most popular models, are the Wasserstein barycenter model \cite{xu2018distilled, schmitz2018wasserstein, werenski2022measure, cheng2021dynamical} which defines a \emph{geometric} average through optimal displacement, and  Schr\"odinger-bridge based formulations which incorporate entropic regularization into the transport process \cite{mallery2025synthesis}. Other examples include \cite{werenski2025linearized, JMLR:v25:23-0708, bai2023linear}. Such models have proved useful whenever one seeks a representative distribution that summarizes several heterogeneous populations. However, in many applications, the observed distributions are not simply different realizations of the same population. Rather, they may arise from a common latent source through a combination of systematic transformation and stochastic dispersion. For instance, multiple cell populations can correspond to different stages of a developmental process, while several image distributions may share the same latent shape but differ through translation, rescaling, and spatial diffusion. In such settings, directly averaging the observed endpoint distributions may fail to recover the underlying latent structure, because the variability introduced by the observation process is itself part of the geometry that should be modeled. This motivates a barycenter construction that can distinguish deterministic transport from stochastic spreading. The recently introduced Schr\"odinger--Bass Bridge (SBB) provides exactly such a mechanism by jointly controlling the drift and the volatility of a stochastic process \citep{henry2026bridging,hasenbichler2026weak}. Unlike the classical Schr\"odinger Bridge, whose volatility is fixed, SBB allows the diffusion component to adapt; unlike Bass martingale transport, it does not impose a zero-drift or convex-order constraint on the endpoints. 
Building on this flexibility, we introduce SBB barycenter models, in which a shared probability measure is connected to the input distributions through SBB bridges. Because the SBB cost is intrinsically asymmetric, the position of the barycenter relative to the input measures carries modeling meaning, leading naturally to encoder, decoder, symmetric, and Markovian formulations.

\paragraph{Main Contributions.} 
\begin{itemize}
\setlength \itemsep{-3pt}
    \item We introduce the encoder, decoder, symmetric, and Markovian SBB barycenter formulations and establish their existence on $\Ptwo(\R^d)$.
    \item We derive the marginal first variations of the SBB objective. Using these we derive a hybrid Wasserstein gradient particle descent scheme that is orientation-specific, to solve these barycenter computation problems efficiently.
    \item For the decoder SBB problem with Gaussian distributions, we present a static Gaussian projection result that reduces the decoder search to Gaussian laws, together with a covariance-action description for Gaussian pairs. We then specialize the particle update to means and covariances.
    \item  Finally, we provide two applications that showcase the utility of the SBB barycenter models: (1) Recovery of 2-D measures from geometric deformation and spatial blur; (2) Ancestor reconstruction and held-out interpolation reconstruction in single-cell RNA sequencing data. We show that SBB models and proposed methods significantly outperform existing barycenter models/methods and baselines. 
\end{itemize}

\section{RELATED WORK}
\label{sec:related}

\paragraph{Wasserstein barycenter and entropic regularization.}

Wasserstein barycenters average probability measures through optimal displacement \citep{agueh2011barycenters}, whose dynamic formulation minimizes kinetic transport energy \citep{benamou2000computational}. Entropic regularization enabled scalable barycenter algorithms, notably the methods of \citet{cuturi2014fast} and the iterative Bregman projections of \citet{benamou2015iterative}. At the same time, entropic regularization modifies the underlying averaging geometry: entropy terms that are constant when both marginals are fixed become nonconstant once one marginal is an optimized variable, and can therefore change the resulting barycenter. On studying EOT barycenter, debiased Sinkhorn barycenters \citep{janati2020debiased} address the resulting entropic bias, while \cite{chizat2023doubly} distinguishes regularization of the transport plans from an additional entropy penalty imposed directly on the barycenter. \cite{mallery2025synthesis} studies the recovery of marginal measure weights through quadratic optimization. 

In the Gaussian setting, Wasserstein barycenters inherit Gaussian structure with Bures--Wasserstein covariance geometry \citep{agueh2011barycenters,nguyenbarycenter}, and related Gaussian closure and closed-form results are known for entropic transport \citep{del2020statistical,mallasto2022entropy,janati2020gaussian,bunne2023gaussian}. Computationally, prior barycenter methods include fixed-point schemes \citep{alvarez2016fixedpoint}, first-order methods with convergence guarantees \cite{altschuler2021averaging, nguyenbarycenter}, and particle descent for regularized barycenters \citep{chizat2023doubly}. 

\paragraph{Schr\"odinger bridges and multi-marginal barycenters.}

The Schr\"odinger problem originates with \citet{schrodinger1931umkehrung} and connects entropy minimization to stochastic control \citep{daipra1991stochastic,leonard2014survey}. For a Brownian reference, the diffusion coefficient is prescribed and the drift is controlled. Multi-marginal extensions already provide stochastic formulations of distributional averaging: \citet{haasler2021multimarginal} study tree-structured transport costs and Schr\"odinger problems; \citet{noble2023tree} develop TreeDSB and apply it to barycenters through a star-shaped construction; and \citet{howard2025tree} extend bridge matching to tree-structured problems and entropic barycentres. Thus, coupling multiple distributions through a shared stochastic representation has attracted substantial attention in prior work.

\paragraph{Bass martingales and weak optimal transport.}
Bass's stochastic-integral construction \citep{bass1983skorokhod} underlies the martingale Benamou--Brenier problem studied by \citet{backhoff2020martingale}. Here the drift is fixed to zero, volatility is optimized, and endpoint feasibility requires convex order. The static formulation belongs to weak optimal transport, where the cost depends on conditional distributions rather than only on pairs of points. General foundations were developed by \citet{gozlan2017kantorovich} and \citet{backhoff2019existence}. This distinction is essential for SBB: its finite-parameter static cost is also a weak transport cost, but allowing nonzero drift removes the martingale convex-order restriction on its endpoints \citep{hasenbichler2026weak}.

\paragraph{Schr\"odinger--Bass Bridge.}
Joint drift and diffusion control belongs to the broader theory of semimartingale optimal transport \citep{tan2013optimal}. Within this framework, \citet{alouadi2026pde} give a one-dimensional PDE derivation of the Schr\"odinger--Bass bridge, while \citet{henry2026bridging} develop its semimartingale formulation, duality, and coupled bridge characterization. \citet{hasenbichler2026weak} establish an equivalent static weak-transport formulation and a dual-ascent computational scheme. On the computational side, \citet{alouadi2026lightsbb} introduce LightSBB-M for sample-based pairwise bridge estimation. These works provide the pairwise theoretical foundations on which our work builds.

\section{SCHR\"ODINGER-BASS BRIDGE (SBB)}
Let $T>0$ be lenghth of time and let $\Omega=C\left([0, T], \mathbb{R}^d\right)$ be the canonical space with a canonical process $X$, i.e. $X_t(\omega)=$ $\omega(t)$ for all $\omega \in \Omega, t \in[0, T]$. We denote by $\mathcal{P} =: \mathcal{P}(\Omega)$ the set of probability measures $\mathbb{P}$ on $\Omega$ under which $X$ has the controlled dynamic
\[
 \dd X_t=\alpha_t\dd t+\sigma_t\dd W_t,
\]
with drift $\alpha_t$ and volatility $\sigma_t$. Its action according to a specific law $\mathbb{P}$ is defined by
\begin{equation}\label{eq:action-main}
 J(\mathbb{P})=\E_{\mathbb{P}}\int_0^T\!\left[\tfrac12\norm{\alpha_t}^2+
 \tfrac\beta2\norm{\sigma_t-\sqrt\varepsilon I_d}_F^2\right]\dd t.
\end{equation}
Let $\Ptwo(\R^d)$ denote the probability measures on $\mathbb{R}^d$ with a finite second moment. For two distributions $\mu, \nu \in \Ptwo(\R^d)$, the Schr\"odinger--Bass cost \citep{henry2026bridging,hasenbichler2026weak} is
\begin{equation}\label{eq:cost-main}
 \sbbcost(\mu,\nu)=\inf_{\mathbb{P} \in \mathcal{P}(\Omega):\,X_0\sim\mu,\,X_T\sim\nu}J(\mathbb{P}).
\end{equation}
See Appendix \ref{app:background} for the comparison table of formulation between Schr\"odinger Bridge, Bass Martingale and SBB. The SBB unifies these two viewpoints by allowing both the drift $(\alpha_t)$ and volatility $(\sigma_t)$ to be controlled, thereby interpolating between entropic diffusion transport and martingale transport. Importantly, $\SBB$ is asymmetric. Indeed, let $H_T\mu=\mu*\N(0,\tau I_d)$ with $\tau=\varepsilon T$ and $*$ is convolution operator. Consider $\alpha_t=0, \sigma_t=\sqrt{\varepsilon}I_d$, leading to control dynamic $dX_t = \sqrt{\varepsilon}\,dW_t, X_0\sim\mu$; which gives $X_T\sim H_T\mu$ and hence, $\SBB(\mu,H_T\mu)=0$. Now reverse the endpoints. If a zero-cost process existed from $H_T\mu$ to $\mu$, then necessarily
$\alpha_t=0, \sigma_t=\sqrt{\varepsilon}I_d$
almost everywhere, since both nonnegative terms in the SBB action would have to vanish. Starting instead from $X_0\sim H_T\mu$, this dynamic gives $X_T \sim H_T(H_T\mu) = H_{2T}\mu \neq \mu$ for every $\mu\in\mathcal P_2(\mathbb R^d)$ whenever $\tau >0$. Therefore the reverse transport $\SBB(H_T\mu, \mu) \neq 0$. Thus SBB is asymmetric, hence not a metric. More properties of $\SBB$ appear in Appendix \ref{app:background}.

\section{PROPOSED SBB BARYCENTERS MODELS}
\label{sec:formulation}
Because the SBB cost is directed, the position of the barycenter variable in \Cref{eq:cost-main} matters. 

Given input measures $\mu_i\in\Ptwo(\R^d)$ and weights $\lambda_i>0$ with $\sum_i\lambda_i=1$, we define different oriented-based Barycenter models
\begin{itemize}
\setlength \itemsep{-3pt}
    \item Encoder Barycenter 
        \begin{equation} 
\rhoenc \in \argmin_{\rho \in \Ptwo(\R^d)}[\Benc(\rho) :=\sum_{i=1}^N\lambda_i\SBB(\mu_i,\rho)],\label{eq:benc-main}
        \end{equation}
    \item Decoder Barycenter 
        \begin{equation}
\rhodec \in \argmin_{\rho \in \Ptwo(\R^d)}[\Bdec(\rho) :=\sum_{i=1}^N\lambda_i\SBB(\rho, \mu_i)],\label{eq:bdec-main}
        \end{equation}
    \item Symmetric Barycenter
        \begin{equation} 
\!\!\!\!\!\!\!\!\!\!\!\!\rhosym \in \argmin_{\rho \in \Ptwo(\R^d)} [\Bsym(\rho) :=\tfrac12\bigl[\Benc(\rho)+\Bdec(\rho)\bigr]].\label{eq:bsym-main}
        \end{equation}
\end{itemize}
An \emph{encoder barycenter} is a common terminal law reached from all input domains. A \emph{decoder barycenter} learns a common source from which the inputs are generated by domain-specific bridges. The \emph{symmetric barycenter} uses both directions and consequently avoids asymmetry. Also, when there are only two marginals, one for the initial and the other for the final time point, we can interpolate the marginal at an intermediate time point as the barycenter of the initial and final marginals (see our application in \Cref{sec:single-cell} about single-cell RNA intermediate-state reconstruction). This motivates us to propose the following interpolation-variant of the barycenter:
\begin{itemize}
\item SBB Markovian Barycenter
{\small
\begin{equation} 
    \!\!\!\!\!\!\!\!\!\!\!\! \rho_{\text {Mar}} \in \argmin _{\rho \in P_2\left(\mathbb{R}^{d}\right)}
    [\mathcal{B}_{\text{mar}}(\rho) :=  \lambda_1  \operatorname{SBB}\left(\mu_1, \rho\right)+ \lambda_2 \operatorname{SBB}\left(\rho, \mu_2\right)], \label{eq:bmarkovian-main} 
\end{equation}
}
\end{itemize}
where $\mu_1$ and $\mu_2$ are the initial and final marginals and $\lambda_1, \lambda_2$ are non-negative weights.

\paragraph{Existence.} We establish that these SBB barycenter formulations are well posed by proving the existence of minimizers. In particular, we have the following result

\begin{theorem}[Existence]\label{thm:existence-main}
    Let $
\mu_1,\ldots,\mu_N\in\mathcal{P}_2(\R^d)$.
Then the decoder barycenter problem \ref{eq:bdec-main} admits at least one minimizer in $\Ptwo(\R^d)$. The same conclusion holds for the encoder \ref{eq:benc-main}, symmetric \ref{eq:bsym-main} and Markovian \ref{eq:bmarkovian-main} barycenter problem. 
\end{theorem}

The full proof is provided in Appendix~\ref{proof:existence}. We present the decoder case in detail, as it is technically more challenging. The encoder, symmetric and Markovian case follows essentially the same argument, as explained in the remark after the proof. The key ingredient is based on this inequality characterizing the bound of SBB using $W_2$
\begin{equation*}\label{eq:coercivity-main}
 \SBB(\rho,\mu_i)\geq
 \frac{W_2^2(\rho,\mu_i)-4\varepsilon dT}{4T+8/\beta}.
\end{equation*}
We leave the question of uniqueness to future work.

\section{A HYBRID WASSERSTEIN GRADIENT PARTICLE DESCENT SCHEME}
\label{sec:algorithm}
The main goal of this section is to design an efficient numerical algorithm to solve for the SBB-Barycenter. Since the optimization problem is posed over probability measures, a natural direction is to employ a particle-based gradient scheme. We represent the current barycenter by $\rho^k=M^{-1}\sum_{m=1}^M\delta_{X_m^k}$ and update the particles $X_m^k$ instead. The update is done on manifold using Wasserstein gradient descent (WGD) (for the background, see \cite{yang2023estimating}). Our key step is to compute the Wasserstein gradient of either $SBB(\mu,\rho)$ or $SBB(\rho,\mu)$ with respect to the variable measure $\rho$. We show that for the SBB functional, its first variation may be expressible in terms of a suitable dual potential. If this dual potential can be computed for each term in the objective, then we can use its gradient to update the particles, yielding a practical particle-based scheme for approximating the barycenter.
\subsection*{First Variation of SBB and Update Rules}

\cite{alouadi2026lightsbb} proved the following dual variational characterization of the pairwise SBB
\begin{equation}\label{eq:sbb-dual}
\small
\SBB(\mu_0,\mu_T)=\sup_{(\psi,v)}\left\{\int\psi(x)\,\mu_T(\dd x)-\int v(0,x)\,\mu_0(\dd x)\right\}.
\end{equation}
The function $v$ solves the associated backward Hamilton--Jacobi--Bellman (HJB) equation with $v(T,\cdot)=\psi$ and curvature constraint $D^2v<\beta I_d$.
The reduced SBB system introduces heat potentials $h_t^*$ and maps $Y_t$ through
\begin{equation*}\label{eq:Phi-def}
\Phi_t(y)=\frac{|y|^2}{2}+\frac{\eps}{\beta}\log h_t^*(y),
\qquad
Y_t=(\nabla\Phi_t)^{-1}.
\end{equation*}
Equivalently, we have the computationally crucial identity
\begin{equation}\label{eq:key-force-identity}
\nabla v^*(t,x)
=\beta\bigl(x-Y_t(x)\bigr)
=\eps\nabla\log h_t^*\bigl(Y_t(x)\bigr).
\end{equation}
To have the information of $Y_t$ or $h_t^*$ and thereby trace back the corresponding dual quantities $v$ and $\psi$, we employ a pairwise SBB solver that provides a representation of the pairwise bridge. Specifically, the solver may return either the inverse map $Y_i(t,x)$ or a score/drift field $s_i(t,y)=\varepsilon \nabla \log h_{i,t}(y)$. In practice, we choose LightSBB \cite{alouadi2026lightsbb}.

In the Appendix \ref{app:particles}, we provide the detailed derivation of the SBB convex subgradient and the resulting first-variation formulas. Under suitable regularity assumptions on the dual solutions, the main results are summarized as follows.
\begin{itemize}
    \item Encoder Marginal First Variation
    $$\nabla_W \SBB(\mu,\rho)(x)=\nabla\psi_{\mu\to\rho}^*(x).$$
    \item Decoder Marginal First Variation
    $$\nabla_W\SBB(\rho, \mu)(x)=-\nabla v^*_{\rho\to\mu}(0,x).$$
\end{itemize}
Note the sign change, which arises because in the decoder case the optimized marginal changes its position in the dual objective.
Using above calculation for $\Benc$ \ref{eq:benc-main}, we have
\begin{equation*}
\frac{\delta \mathcal B_{\mathrm{enc}}}{\delta\rho}(x)
=\sum_{i=1}^N\lambda_i\psi_i^*(x),
\end{equation*}
where $\psi_i^*$ is the terminal dual potential for the pairwise bridge $\mu_i\to\rho$.
Using the Moreau-envelope representation, $\beta\bigl(x-Y_T^{(i)}(x)\bigr) \in \partial\psi_i^*(x)$. So the explicit particles update for any particle $m = 1:M$ with step size $\eta_k$ is
\begin{equation}\label{eq:encoder-particle-update}
\boxed{
X_m^{k+1}
=X_m^k
-\eta_k\sum_{i=1}^N\lambda_i\beta\bigl(X_m^k-Y_T^{(i,k)}(X_m^k)\bigr).
}
\end{equation}

Using above calculation for $\Bdec$ \ref{eq:bdec-main}, we have
\begin{equation*}
\frac{\delta \mathcal{B}_{\mathrm{dec}}}{\delta\rho}(x)
=-\sum_{i=1}^N\lambda_i v_i^*(0,x),
\end{equation*}
where $v_i^*$ is the HJB value for the pairwise bridge $\rho\to\mu_i$.
Thus the descent velocity is $\sum_{i=1}^N\lambda_i\beta\bigl(x-Y_0^{(i)}(x)\bigr)$
The explicit particle update is
\begin{equation}\label{eq:decoder-particle-update}
\boxed{
X_m^{k+1}
=X_m^k
+\eta_k\sum_{i=1}^N\lambda_i\beta\bigl(X_m^k-Y_{0}^{(i,k)}(X_m^k)\bigr).
}
\end{equation}

For $\Bsym$ \ref{eq:bsym-main}, the natural smoothed descent velocity is the average of the decoder and encoder velocities
{\footnotesize
\begin{equation*}\label{eq:sym-velocity}
\mathfrak v_{\mathrm{sym}}(x)
=\frac12\sum_{i=1}^N\lambda_i
\left[
\beta\bigl(x-Y_0^{(i,k),\mathrm{dec}}(x)\bigr)
-
\beta\bigl(x-Y_{T}^{(i,k),\mathrm{enc}}(x)\bigr)
\right].
\end{equation*}
Thus
\begin{equation}\label{eq:sym-update}
\boxed{
X_m^{k+1}=X_m^k+\eta_k\mathfrak v_{\mathrm{sym}}(X_m^k).
}
\end{equation}}

Finally, for the Markovian objective $\mathcal{B}_{\text{mar}}$ \ref{eq:bmarkovian-main}
\begin{equation*}
  \mathfrak v_{\mathrm{mar}}(x)
  =
  \lambda_2\beta\bigl(x-Y_0^{(2,k),\mathrm{dec}}(x)\bigr)
  -
  \lambda_1\beta\bigl(x-Y_T^{(1,k),\mathrm{enc}}(x)\bigr),
  \label{eq:sbb-markovian-velocity}
\end{equation*}
thus
\begin{equation}\label{eq:markovian-update}
\boxed{
X_m^{k+1}=X_m^k+\eta_k\mathfrak v_{\mathrm{mar}}(X_m^k).
}
\end{equation}

In practice, since the encoder first variation is represented by the terminal dual potential, its spatial derivative may be less regular at $t=T$. For numerical stability, we therefore may evaluate the encoder duality at $T-\delta_k$, with $\delta_k>0$ small, where the HJB provides additional smoothing. This regularization could be applied whenever an encoder component enters the update, including the symmetric and Markovian barycenters.

\begin{algorithm}[t]
\caption{{\scriptsize SBB Barycenter particle descent (outer loop)}}\label{alg:main}
\begin{algorithmic}[1]
\Require Input $\{\mu_i\}_{i=1}^n$, initial particles $X_m^0$, step sizes $\eta_k$, pairwise bridge solver (which accounts for inner loop)
\For{$k=0,\ldots,K-1$}
 \State $\rho^k\gets M^{-1}\sum_m\delta_{X_m^k}$
 \For{each $i$}
  \State Bridge fit $\rho^k\to\mu_i$ for decoder terms
  \State Bridge fit $\mu_i\to\rho^k$ for encoder terms
 \EndFor
 \State Update $X_m^{k+1}$ by \ref{eq:encoder-particle-update}, \ref{eq:decoder-particle-update}, \ref{eq:sym-update} or \ref{eq:markovian-update} for all $m$
\EndFor
\State \Return $M^{-1}\sum_m\delta_{X_m^K}$
\end{algorithmic}
\end{algorithm}

Algorithm~\ref{alg:main} displays the common outer loop. Appendix \ref{app:pseudocode} contains the full pseudocode. The Gaussian specialization in the next section gives a lower-dimensional implementation when the problem reduces to optimizing its mean and covariance.

\section{THE GAUSSIAN CASE}
\label{sec:gaussian}
The goal of this section is to study the geometry of the SBB barycenter when the input measures are Gaussian, i.e. $\mu_i=\mathcal N(m_i,\Sigma_i)$, and, in particular, to determine whether the barycenter can also be chosen Gaussian.

Gaussian closure is natural in related transport models. For the Schr\"odinger Bridge, this behavior is consistent with its static entropic optimal transport formulation \cite{leonard2014survey} together with the Gaussian closure of Wasserstein barycenters \cite{agueh2011barycenters}. Meanwhile, for Bass martingale, in Gaussians setting, due to convex order constraint, a common Gaussian source can exist only when all terminal means coincide, say $m_1=\cdots=m_N=:m$. In that case, $\mathcal N(m,S)$ is feasible precisely when $0\preceq S\preceq \Sigma_i$ for all $i$, and the covariance is obtained from the finite-dimensional convex problem $\min_{\substack{0 \prec S\preceq\Sigma_i,\ \forall i}}
\frac12\sum_{i=1}^N
\lambda_i
\left\|
(\Sigma_i-S)^{1/2}
-\sqrt{\varepsilon T}\,I
\right\|_F^2.$

The SBB setting is different from both cases. It is not a linear combination of the Schr\"odinger and Bass problems: unlike Bass transport, SBB allows nonzero drift and therefore does not impose convex order on the endpoints, while unlike the Schr\"odinger Bridge, the adaptive volatility in SBB prevents it from being reduced to an ordinary point-to-point transport problem. Our main result in this section is that, nevertheless, the decoder SBB barycenter of Gaussian marginals admits a Gaussian solution; see \Cref{thm:projection-main}.

To prove this result, we work with a static representation rather than directly with the dynamic control problem. This viewpoint also clarifies the relation with the two limiting geometries. The Schr\"odinger Bridge admits a static entropic transport formulation \cite{leonard2014survey}, whereas the Bass martingale problem leads to a weak transport problem whose cost depends on conditional distributions \cite{backhoff2020martingale}. For SBB, building on the measure-valued static formulation of \cite{hasenbichler2026weak}, we derive an equivalent three-variable representation in Proposition \ref{prop:triple-static}. The variables have distinct roles: $X$ is the physical initial state, $Y$ is a hidden terminal state, and $Z$ is the physical terminal state. This representation is the key tool for the Gaussian projection argument.

\paragraph{Three-variable static representation.} 
Define \(\mathcal A(\mu,\nu)\) to be the set of all probability
measures
\[
\Gamma\in
\mathcal P_2
\bigl(
\mathbb R^d\times\mathbb R^d\times\mathbb R^d
\bigr)
\]
on variables \((X,Y,Z)\) satisfying $
\Gamma_X=\mu,
\Gamma_Z=\nu$, 
and
\begin{equation*}
\mathbb E_\Gamma[Y\mid X]=\mathbb E_\Gamma[Z\mid X]
\qquad
\Gamma\text{-a.s.}
\label{eq:conditional-mean-constraint}
\end{equation*}
Equivalently, if $
\Gamma(dx\,dy\,dz)
=
\mu(dx)\Gamma_x(dy\,dz)$,
then
\begin{equation*}
\int_{\mathbb R^d\times\mathbb R^d}
(y-z)\,\Gamma_x(dy\,dz)
= 0
\qquad
\mu\text{-a.e. }x.
\label{eq:conditional-mean-kernel}
\end{equation*}

For \(\Gamma\in\mathcal A(\mu,\nu)\), with the notation $\gamma_{x,\tau}=\N(x,\tau I_d)$, define
\begin{equation}
\mathcal J_{\beta,\varepsilon,T}(\Gamma)
:=
\varepsilon
H\!\left(
\Gamma_{XY}
\,\middle|\,
\mu(dx)\gamma_{x,\tau}(dy)
\right)
+
\frac{\beta}{2}
\int \|y-z\|^2\,\Gamma(dx\,dy\,dz).
\label{eq:J-static}
\end{equation}

\begin{proposition}[Exact triple representation]
\label{prop:triple-static}
For every \(\mu,\nu\in\mathcal P_2(\mathbb R^d)\),
\begin{equation}
\boxed{
\SBB_{\beta,\varepsilon,T}(\mu,\nu) = \inf_{\Gamma\in\mathcal A(\mu,\nu)} \mathcal J_{\beta,\varepsilon,T}(\Gamma).
}
\label{eq:triple-representation}
\end{equation}
\end{proposition}

Next, we present our main result, \Cref{thm:projection-main}, which shows that replacing the source marginal measure by its Gaussian projection can only decrease the SBB cost. This projection principle is the key ingredient for establishing the existence of a Gaussian solution to the decoder SBB barycenter problem in Corollary~\ref{cor:gaussian-main}.

\begin{theorem}[Gaussian projection]\label{thm:projection-main}
Let $\nu=\N(m_\nu,\Sigma_\nu)$ be Gaussian, with $\Sigma_\nu\succeq0$. For $\mu\in\Ptwo(\R^d)$, let $\mu^G:=\N(\E_\mu[X],\Cov_\mu(X))$ be a Gaussian which has the same first and second moment as $\mu$. Then, for every $\beta,\varepsilon,T>0$,
\begin{equation*}\label{eq:projection-main}
 \sbbcost(\mu^G,\nu)\leq\sbbcost(\mu,\nu).
\end{equation*}
\end{theorem}
\begin{corollary}\label{cor:gaussian-main}
When all
$\mu_i=\mathcal N(m_i,\Sigma_i), \quad \Sigma_i\succeq0, \quad i=1,\dots,N
$ are Gaussians; we have for every $\rho\in\mathcal P_2(\mathbb R^d)$,
\begin{equation*}
\mathcal \Bdec(\rho^G) \leq \mathcal \Bdec(\rho).
\label{eq:barycenter-Gaussianization}
\end{equation*}
Consequently,
\begin{equation*}
\inf_{\rho\in\mathcal P_2(\mathbb R^d)}
\mathcal \Bdec(\rho)
= \inf_{\substack{m\in\mathbb R^d, S\in\mathbb S_+^d
}}
\mathcal \Bdec\bigl(\mathcal N(m,S)\bigr),
\label{eq:barycenter-Gaussian-reduction}
\end{equation*}
or problem \ref{eq:bdec-main} admits a Gaussian minimizer.
\end{corollary}
Appendix \ref{app:main_proof_gaussian} contains the complete triple-representation, entropy-lemma, projection, and corollary proofs. By establishing that the decoder SBB barycenter of Gaussian marginals is itself Gaussian, the barycenter problem reduces to a finite-dimensional optimization over its mean and covariance. Building on this reduction, we propose a specialization of Algorithm~1 tailored to the Gaussian setting. 

\paragraph{Gaussian Specialization of SBB Decoder Barycenter Particle Descent.} To exploit this reduction computationally, we first characterize the geometry of the pairwise SBB problem between two Gaussian measures. This pairwise problem will subsequently serve as the inner problem of the Gaussian barycenter algorithm. Let $\mu=\mathcal N(m_\mu,\Sigma_\mu), \nu=\mathcal N(m_\nu,\Sigma_\nu),$ with $\Sigma_\mu,\Sigma_\nu\succ0$. For $\Sigma\succ0$ and a symmetric matrix $U$, let $K$ denote the unique symmetric solution of the Lyapunov equation $K\Sigma+\Sigma K=U$, and define
$\Gamma_\Sigma(U):=\Tr(K\Sigma K)$. We then introduce the covariance action
\begin{equation}
\footnotesize
\label{eq:covaction-main}
\begin{aligned}
& \mathcal G_{\beta,\varepsilon,T}(\Sigma_0,\Sigma_1)
\\ &:=
\inf_{\Sigma_t,A_t}\int_0^T
\Bigg[
&\frac12
\Gamma_{\Sigma_t}\bigl(\dot\Sigma_t-A_t\bigr) +
\frac{\beta}{2}
\left\|
A_t^{1/2}-\sqrt{\varepsilon}I_d
\right\|_F^2
\Bigg]\dd t,
\end{aligned}
\end{equation}
where $\Sigma(0)=\Sigma_\mu, \Sigma(T)=\Sigma_\nu, A_t\succeq0$. The following result separates the Gaussian SBB geometry into a Euclidean mean component and a covariance component.
\begin{theorem}[Gaussian covariance-action decomposition]
\label{thm:covaction-main} 
We have
\begin{equation}
\label{eq:gauss-decomposition-main}
\mathsf{SBB}_{\beta,\varepsilon,T}(\mu,\nu)
=
\frac{\|m_\mu-m_\nu\|^2}{2T}
+
\mathcal G_{\beta,\varepsilon,T}(\Sigma_\mu,\Sigma_\nu).
\end{equation}
\end{theorem}

The complete proof is given in Appendix~\ref{app:main_proof}. Importantly, the proof also characterizes the attaining Gaussian dynamics. In particular, the optimal bridge can be represented by an affine Gaussian process
\begin{equation*}
\label{eq:affine-sde-main}
\dd X_t
=
\dot m_t\,\dd t
+
K_t(X_t-m_t)\,\dd t
+
A_t^{1/2}\,\dd W_t,
\end{equation*}
where $m_t=\left(1-\frac{t}{T}\right)m_\mu+\frac{t}{T}m_\nu$, and $K_t\Sigma_t+\Sigma_t K_t = \dot\Sigma_t-A_t$. By the affine Gaussian dynamics above, the dual variable in \Cref{sec:algorithm} takes the form
\begin{equation}
\label{eq:gaussian-drift-main}
\nabla v_i^k(0,x)
=
\frac{m_i-m^k}{T}
+
K_{i,0}^k(x-m^k),
\end{equation}
where $K_{i,0}^k$ is obtained from the corresponding pairwise Gaussian SBB bridge. Substituting \eqref{eq:gaussian-drift-main} into \ref{eq:decoder-particle-update} gives
\begin{equation*}
X^{k+1} = m^k + \frac{\eta_k}{T}(\bar m-m^k) + L_k(X^k-m^k).
\label{eq:gaussian-affine-update}
\end{equation*}
with $\bar m = \sum_{i=1}^N\lambda_i m_i, \quad \bar K_0^k:=\sum_{i=1}^N\lambda_iK_{i,0}^k, \quad L_k:=I_d+\eta_k\bar K_0^k$. 
Taking expectation yields the mean update
\begin{equation*}
\label{eq:mean-update-main}
m^{k+1}=m^k+\frac{\eta_k}{T}(\bar m-m^k).
\end{equation*}
Hence,
$X^{k+1}-m^{k+1}=L_k(X^k-m^k)$,
and therefore the covariance update is according to
\begin{equation*}
\label{eq:cov-update-main}
S^{k+1}=L_kS^kL_k^\top.
\end{equation*}
The algorithm based on this analysis is given in Appendix \ref{alg:gaussian-barycenter}. Compared with the general algorithm, the Gaussian specialization replaces the $M$ barycenter particles by a single mean vector and covariance matrix. The pairwise Gaussian SBB problems remain as inner problems, but their outputs are summarized via matrices $K_{i,0}^k$. 

\paragraph{Remark.} Whether Gaussian closure holds for the encoder barycenter remains an open theoretical question. However, the anisotropic Gaussian experiment in \Cref{sec:toy-experiment}, with results shown in Appendix \ref{app:experiments}, provides numerical evidence suggesting that this property may fail.

\section{APPLICATIONS}
\label{sec:experiments}
Unlike Wasserstein transport, the SBB model can account for stochastic spreading through its volatility component. Our experiments are organized around this interpretation. 
We now exploit the asymmetry and directionality of SBB. We begin with a simple particle example illustrating why the direction of the SBB barycenter matters in \Cref{sec:toy-experiment}. We then consider two complementary latent-recovery problems for handwritten digits in \Cref{sec:mnist-experiment} and single-cell RNA sequencing in \Cref{sec:single-cell}. Detailed settings of all experiments are in Appendix  \ref{app:experiments}. All experiments are run on a NVIDIA H200.

\subsection{Directionality of the SBB Barycenter}
\label{sec:toy-experiment}
Since $\SBB(\mu,\nu) \neq \SBB(\nu,\mu)$ in general (Section \ref{sec:formulation}), the encoder, decoder, and symmetric formulations need not produce the same barycenter. Figure~\ref{fig:toy-dec-app}-\ref{fig:toy-sym-app} in Appendix \ref{app:experiments} visualize this effect on two-dimensional particles using three anisotropic Gaussian distributions as marginals. The same collection of input distributions is used in all three panels, but the resulting barycenter particles differ in both shape and spread. In this particular example, the decoder barycenter is more concentrated, while the encoder and symmetric formulations produce visibly different geometries. The figure is intended only to illustrate the directional nature of SBB; it should not be interpreted as a comparison of convergence rates or numerical accuracy.

This observation motivates the choice of direction in the applications below. Our objective is to infer a common latent distribution from several observed, corrupted domains. The natural generative picture is therefore $\rho^\star \longrightarrow \nu_1,\ldots,\nu_N$, where $\rho^\star$ is an unknown source and the $\nu_i$ are its observed endpoints. Consequently, the decoder barycenter $\rhodec$ is the formulation aligned with the latent-recovery task.

\subsection{Recovering a Sharp Spatial Measure}
\label{sec:mnist-experiment}

This application is inspired by the spatial barycenter experiment of \citet{cuturi2014fast}. Here an image is interpreted as a probability measure on the pixel plane: after normalizing its intensity, the pixel values specify how probability mass is distributed over spatial locations. 

Starting from a clean latent digit $\rho_c^\star$, we generate $N=32$ observations of the form $\nu_{c,i} = (T_i)_\#\rho_c^\star * \N(0,\Omega_i)$, where $T_i$ applies deterministic geometric transformations such as translation and scaling, with rotation included as an ablation, while the Gaussian convolution introduces spatial diffusion. Moreover, each observation receives a different anisotropic blur. The clean source $\rho_c^\star$ is never supplied to any barycenter method. It is retained only as ground truth for evaluation. We report average results across the ten test images using $W_2$ distance to the latent source, rendered pixel error, and SSIM \cite{nilsson2020understanding}.

\begin{table}[htp!]
\caption{MNIST latent-source recovery. Report paired averages across latent digits.}
\label{tab:mnist-main}
\centering
\begin{tabular}{lccc}
\hline
Barycenter Method
& $W_2$ error $\downarrow$
& Pixel $\ell_1$ error $\downarrow$
& SSIM $\uparrow$ \\
\hline
Arithmetic
& $3.21 \pm 0.20$
& $1.35 \pm 0.09$
& $0.9808 \pm 0.005$ \\

Wasserstein
& $1.01 \pm 0.11$
& $0.62 \pm 0.07$
& $0.9964 \pm 0.0013$ \\

RegWB 
& $2.08 \pm 0.04$
& $1.15 \pm 0.01$
& $0.9727 \pm 0.0015$ \\

Debiased Sinkhorn
& $0.99 \pm 0.11$
& $0.62 \pm 0.06$
& $0.9964 \pm 0.0012$ \\

Decoder SBB
& \textcolor{red}{$0.79 \pm 0.09$}
& \textcolor{red}{$0.45 \pm 0.07$}
& \textcolor{red}{$0.9975 \pm 0.0010$} \\
\hline
\end{tabular}
\end{table}

We compare against Euclidean arithmetic mean, Wasserstein/Bures barycenters, regularized Wasserstein barycenter (RegWB) \cite{noble2023tree} and debiased Sinkhorn barycenters \cite{janati2020debiased}. For all methods, preprocessing, data splits, and corruption draws are kept identical. Table~\ref{tab:mnist-main} confirms the qualitative results in \Cref{fig:mnist_visualization} in Appendix \ref{app:experiments}. Wasserstein and debiased Sinkhorn substantially outperform the arithmetic mean, while decoder SBB achieves the best performance on all three metrics. In particular, SBB reduces the $W_2$ error and pixel $\ell_1$ error notably, while increasing SSIM. This suggests that the main gain of SBB lies in recovering a sharper and better localized latent spatial measure.

\subsection{Single-cell developmental trajectory reconstruction}
\label{sec:single-cell}

Single-cell RNA sequencing provides population snapshots rather than trajectories of individual cells, since cells are destroyed during measurement \cite{kester2018single}. This makes developmental reconstruction naturally a distributional problem: at each time $t$, the cell population is represented by an empirical measure $P_t$ in gene-expression space, and the goal is to infer how these distributions evolve. We consider the Waddington-OT iPSC reprogramming dataset ~\cite{schiebinger2019optimal}, which contains a dense 18-day time course of mouse embryonic fibroblast reprogramming. Cells were collected at 39 time points approximately every 12 hours (every 6 hours between days 8 and 9), yielding more than $2.5\times 10^5$ high-quality cells. The original work infers ancestor--descendant couplings and shows that iPSC, epithelial, trophoblast, and neural lineages share a common MET-like ancestry until approximately day 8. 

\paragraph{Recovering a common ancestor distribution.}
Motivated by this finding, we test whether a decoder SBB barycenter of late-stage populations resembles the observed MET population around day 8, which serves as a biological reference for evaluating population-level reconstruction. In the serum condition, we take late-stage iPSC, epithelial, trophoblast, and neural populations as descendant measures $\nu_1,\ldots,\nu_K$ and compute their $\rhodec$. The recovered barycenter is interpreted as a population-level common ancestor rather than as a set of individually identified ancestral cells. Importantly, the observed day-8 MET population is \emph{not} used to fit the barycenter; it is retained only for evaluation. We compare the recovered distribution with the target population using $W_2$ for distance between two distributions, MMD for differences in location, spread, and shape in PCA space; cell-type TV for differences in cell-type proportions; signature MAE for differences in population-level MET/epithelial and pluripotency activity. Baselines are the arithmetic mixture of descendant samples, $W_2$, RegWB and Waddington-OT \cite{schiebinger2019optimal} barycenters. Results are summarized as mean and standard deviation across five seeds, with identical endpoint and target samples across methods within each seed. 

{
\begin{table}[!htp]
\centering
\caption{Ancestor reconstruction results. Lower values indicate better agreement for all four metrics. Serum: a culture medium containing serum, which provides components that support cell growth. Serum + stromal: negative-control experiment that adds the stromal population to the four serum-condition descendant populations}
\scriptsize
\resizebox{\linewidth}{!}{%
\begin{tabular}{lrrrr}
\toprule
Barycenter Method & $W_2\downarrow$ & MMD$\downarrow$ &
Cell-type TV$\downarrow$ & Signature MAE$\downarrow$ \\
\midrule
\multicolumn{5}{l}{\textbf{serum condition}} \\
\midrule
Arithmetic
& $19.92 \pm 0.31$
& $0.235 \pm 0.003$
& $0.073 \pm 0.006$
& $0.080 \pm 0.012$ \\
Wasserstein
& $15.99 \pm 0.03$
& $0.241 \pm 0.003$
& $0.068 \pm 0.016$
& $0.078 \pm 0.011$ \\
RegWB
& $16.06 \pm 0.10$ 
& $0.515 \pm 0.002$
& $0.091 \pm 0.007$
& $0.075 \pm 0.012$ \\
Waddington-OT
& $15.87 \pm 0.05$
& $0.254 \pm 0.003$
& $0.070 \pm 0.010$
& $0.086 \pm 0.013$ \\
SBB Decoder (ours)
& \textcolor{red}{$15.85 \pm 0.05$}
& \textcolor{red}{$0.214 \pm 0.003$}
& \textcolor{red}{$0.064 \pm 0.013$}
& \textcolor{red}{$0.068 \pm 0.004$} \\
\midrule
\multicolumn{5}{l}{\textbf{serum + stromal condition}} \\
\midrule
Arithmetic
& $21.00 \pm 0.12$
& $0.254 \pm 0.008$
& $0.068 \pm 0.003$
& $0.070 \pm 0.007$ \\
Wasserstein
& $16.70 \pm 0.13$
& $0.261 \pm 0.003$
& \textcolor{red}{$0.052 \pm 0.008$}
& $0.068 \pm 0.004$ \\
RegWB
& $16.68 \pm 0.07$
& $0.266 \pm 0.004$
& $0.058 \pm 0.020$
& $0.073 \pm 0.014$ \\
Waddington-OT
& $16.42 \pm 0.05$
& $0.271 \pm 0.005$
& $0.060 \pm 0.009$
& $0.089 \pm 0.016$ \\
SBB Decoder (ours)
& \textcolor{red}{$16.40 \pm 0.04$}
& \textcolor{red}{$0.248 \pm 0.001$}
& $0.064 \pm 0.012$
& \textcolor{red}{$0.059 \pm 0.007$} \\
\bottomrule
\end{tabular}%
}
\label{tab:recover_ancestor}
\end{table}
}

As shown in Table~\ref{tab:recover_ancestor}, the decoder SBB barycenter achieves the best overall ancestor reconstruction.  It obtains the lowest error on all four metrics in the serum condition and the lowest $W_2$, MMD and Signature MAE under stromal. While some baselines give slightly better Cell-type TV under stromal, SBB more consistently recovers the global distributional geometry and gene-signature activity of the held-out ancestor.

\paragraph{Intermediate-state reconstruction.} We evaluate whether SBB barycenters can reconstruct intermediate cell-state distributions from two temporal endpoints separated by \emph{one, two, or three days}. Target-time cells are excluded from transport fitting. For three consecutive times $t_-<t_0<t_+$, we remove all cells at mid-point time $t_0$ and approximate the barycenters. For example, a one-day interval uses cells from days 1 and 2 to reconstruct the distribution at day 1.5, while a two-day interval uses days 1 and 3 to reconstruct day 2. We report means and standard deviations aggregated over all temporal windows of each duration and five random seeds. Each run uses 512 cells per endpoint, 512 reconstructed particles, and 512 target cells, all represented in a shared, transductively fitted 30-dimensional PCA space.

In addition to the SBB decoder, we also evaluate SBB-Markovian barycenter \ref{eq:bmarkovian-main}. Baselines include linear arithmetic, Wasserstein barycenter, RegWB, Waddington-OT barycenter and TrajectoryNet \cite{tong2020trajectorynet}. We report $W_2$ distance, MMD, cell-type TV, gene-signature MAE. Due to space constraints, \Cref{tab:serum_2i_1day_comparison} presents the quantitative results for one-day intervals. Appendix \ref{app:experiments} provides quantitative results for two-day and three-day intervals, together with qualitative examples.

\begin{table}[htbp]
\centering
\small
\caption{Intermediate-state reconstruction one-day intervals result. 2i: a serum-free culture medium containing two inhibitors. }
\label{tab:serum_2i_1day_comparison}
\resizebox{1.0\linewidth}{!}{%
\begin{tabular}{lcccc}
\toprule
Interpolation Method & $W_2\downarrow$ & MMD$\downarrow$ & Cell-type TV$\downarrow$ & Signature MAE$\downarrow$ \\
\midrule
\multicolumn{5}{l}{\textbf{serum condition}} \\
\midrule
Arithmetic mixture
& $11.9916 \pm 0.1855$
& $0.4044 \pm 0.0120$
& $0.0588 \pm 0.0289$
& $0.0364 \pm 0.0062$ \\
Wasserstein
& $10.9134 \pm 0.2175$
& $0.4189 \pm 0.0155$
& $0.0859 \pm 0.0284$
& $0.0359 \pm 0.0008$ \\
RegWB
& $11.0323 \pm 0.2059$ 
& $0.4154 \pm 0.0151$
& $0.0892 \pm 0.0171$
& $0.0370 \pm 0.0029$ \\
TrajectoryNet
& $14.6255 \pm 0.2413$ 
& $0.5539 \pm 0.0123$
& $0.0658 \pm 0.0119$
& \textcolor{red}{$0.0329 \pm 0.0043$} \\
Waddington-OT
& $10.6621 \pm 0.0752$
& $0.4034 \pm 0.0080$
& $0.2383 \pm 0.0491$
& $0.0385 \pm 0.0036$ \\
SBB Decoder (ours)
& $10.3515 \pm 0.1308$
& $0.3778 \pm 0.0118$
& \textcolor{red}{$0.0534 \pm 0.0336$}
& $0.0375 \pm 0.0097$ \\
SBB Markovian (ours)
& \textcolor{red}{$10.1848 \pm 0.1203$}
& \textcolor{red}{$0.3408 \pm 0.0106$}
& $0.0553 \pm 0.0329$
& $0.0349 \pm 0.0072$ \\
\midrule
\multicolumn{5}{l}{\textbf{2i condition}} \\
\midrule
Arithmetic mixture
& $11.9001 \pm 0.6598$
& $0.1731 \pm 0.0196$
& $0.1536 \pm 0.0530$
& $0.0882 \pm 0.0422$ \\
Wasserstein
& $11.3029 \pm 0.7513$
& $0.1797 \pm 0.0234$
& $0.1745 \pm 0.0368$
& $0.1146 \pm 0.0458$ \\
RegWB
& $11.5855 \pm 0.7610$
& $0.1773 \pm 0.0247$
& $0.1816 \pm 0.0508$
& $0.1128 \pm 0.0392$ \\
TrajectoryNet
& $12.3068 \pm 0.4723$
& $0.1857 \pm 0.0110$
& $0.1700 \pm 0.0175$
& $0.1037 \pm 0.0071$ \\
Waddington-OT
& $12.6290 \pm 1.2882$
& $0.2239 \pm 0.0400$
& $0.1901 \pm 0.0587$
& $0.1024 \pm 0.0323$ \\
SBB Decoder (ours)
& \textcolor{red}{$10.9584 \pm 0.7523$}
& \textcolor{red}{$0.1572 \pm 0.0236$}
& \textcolor{red}{$0.1530 \pm 0.0536$}
& $0.0791 \pm 0.0384$ \\
SBB Markovian (ours)
& $11.5364 \pm 0.5676$
& $0.1698 \pm 0.0191$
& $0.1536 \pm 0.0530$
& \textcolor{red}{$0.0684 \pm 0.0374$} \\
\bottomrule
\end{tabular}%
}
\end{table}

Table~\ref{tab:serum_2i_1day_comparison} demonstrates that the SBB formulations remain competitive when reconstructing an unobserved intermediate population from its temporal endpoints. 
In the serum condition, SBB Markovian achieves the lowest $W_2$ and MMD errors and also provides strong agreement in gene-signature activity, while SBB Decoder remains competitive in cell-type composition. 
This behavior is consistent with the construction of the Markovian objective: an intermediate population simultaneously acts as the terminal state of the preceding bridge and the initial state of the subsequent bridge, so exploiting both SBB directions provides additional temporal information. Under the 2i condition, SBB Decoder instead gives the lowest $W_2$ and MMD errors, whereas SBB Markovian remains competitive on the biological metrics.  Thus, neither directional variant uniformly dominates across all criteria; rather, the results indicate that the decoder formulation is particularly effective at matching the geometry of the held-out distribution, while the Markovian formulation can better exploit the two-sided temporal structure when reconstructing intermediate states.

\section{CONCLUSION}
\label{sec:conclusion}
\vspace{-0.1in}
We introduced Schr\"odinger--Bass Bridge barycenters as a directed framework for averaging measure-valued data while accounting jointly for systematic transport and stochastic dispersion. Exploiting the intrinsic asymmetry of SBB, we formulated encoder, decoder, symmetric, and Markovian barycenters. We develop the theory of existence, a Wasserstein gradient particle scheme, and Gaussian analysis. The spatial-MNIST and single-cell RNA experiments show that our method improves latent recovery beyond different baselines. Future work includes establishing convergence guarantees for the particle scheme and exploring applications of SBB barycenters (including encoder and symmetric model also) in a broader range of domains.

\section{Acknowledgment}
\vspace{-0.1in}
This work is supported via NSF under Cooperative Agreement PHY-2019786 and PHY-2525568. Shuchin Aeron would also like to acknowledge funding via NSF DMS 2309519.

\bibliography{references}

@inproceedings{yang2023estimating,
 author={Yang, Yibo and Eckstein, Stephan and Nutz, Marcel and Mandt, Stephan},
 title={Estimating the Rate-Distortion Function by {Wasserstein} Gradient Descent},
 booktitle={Advances in Neural Information Processing Systems},volume={36},pages={2768--2794},year={2023},url={https://arxiv.org/abs/2310.18908}}

@book{billingsley2013convergence,
  title={Convergence of Probability Measures},
  author={Billingsley, Patrick},
  year={2013}, publisher={John Wiley \& Sons}}

@inproceedings{schrodinger1932theorie,
  title={Sur la th{\'e}orie relativiste de l'{\'e}lectron et l'interpr{\'e}tation de la m{\'e}canique quantique},
  author={Schr{\"o}dinger, Erwin},
  booktitle={Annales de l'institut Henri Poincar{\'e}},
  volume={2},
  number={4},
  pages={269--310},
  year={1932}
}

@article{conze2021bass,
  title={Bass construction with multi-marginals: Lightspeed computation in a new local volatility model},
  author={Conze, Antoine and Henry-Labordere, Pierre},
  journal={Available at SSRN 3853085},
  year={2021}
}

@article{agueh2011barycenters,
  title={Barycenters in the Wasserstein space},
  author={Agueh, Martial and Carlier, Guillaume},
  journal={SIAM Journal on Mathematical Analysis},
  volume={43},
  number={2},
  pages={904--924},
  year={2011},
  publisher={SIAM}
}

@article{benamou2000computational,
  author  = {Benamou, Jean-David and Brenier, Yann},
  title   = {A Computational Fluid Mechanics Solution to the
             {Monge--Kantorovich} Mass Transfer Problem},
  journal = {Numerische Mathematik},
  year    = {2000},
  volume  = {84},
  pages   = {375--393},
  doi     = {10.1007/s002110050002}
}

@inproceedings{cuturi2014fast,
  author    = {Cuturi, Marco and Doucet, Arnaud},
  title     = {Fast Computation of {Wasserstein} Barycenters},
  booktitle = {Proceedings of the 31st International Conference
               on Machine Learning},
  series    = {Proceedings of Machine Learning Research},
  volume    = {32},
  pages     = {685--693},
  year      = {2014},
  url       = {https://proceedings.mlr.press/v32/cuturi14.html}
}

@article{benamou2015iterative,
  author  = {Benamou, Jean-David and Carlier, Guillaume
             and Cuturi, Marco and Nenna, Luca
             and Peyr{\'e}, Gabriel},
  title   = {Iterative {Bregman} Projections for Regularized
             Transportation Problems},
  journal = {SIAM Journal on Scientific Computing},
  year    = {2015},
  volume  = {37},
  number  = {2},
  pages   = {A1111--A1138},
  doi     = {10.1137/141000439}
}

@inproceedings{janati2020debiased,
  author    = {Janati, Hicham and Cuturi, Marco
               and Gramfort, Alexandre},
  title     = {Debiased {Sinkhorn} Barycenters},
  booktitle = {Proceedings of the 37th International Conference
               on Machine Learning},
  series    = {Proceedings of Machine Learning Research},
  volume    = {119},
  pages     = {4692--4701},
  year      = {2020},
  url       = {https://proceedings.mlr.press/v119/janati20a.html}
}

@misc{chizat2023doubly,
  author        = {Chizat, L{\'e}na{\"i}c},
  title         = {Doubly Regularized Entropic {Wasserstein}
                   Barycenters},
  year          = {2023},
  eprint        = {2303.11844},
  archivePrefix = {arXiv},
  url           = {https://arxiv.org/abs/2303.11844}
}

@article{schrodinger1931umkehrung,
  author  = {Schr{\"o}dinger, Erwin},
  title   = {{\"U}ber die Umkehrung der Naturgesetze},
  journal = {Sitzungsberichte der Preussischen Akademie der
             Wissenschaften, Physikalisch-mathematische Klasse},
  year    = {1931},
  pages   = {144--153}
}

@article{daipra1991stochastic,
  author  = {Dai Pra, Paolo},
  title   = {A Stochastic Control Approach to Reciprocal
             Diffusion Processes},
  journal = {Applied Mathematics and Optimization},
  year    = {1991},
  volume  = {23},
  pages   = {313--329},
  doi     = {10.1007/BF01442404}
}

@article{leonard2014survey,
  author  = {L{\'e}onard, Christian},
  title   = {A Survey of the {Schr{\"o}dinger} Problem and
             Some of Its Connections with Optimal Transport},
  journal = {Discrete and Continuous Dynamical Systems},
  year    = {2014},
  volume  = {34},
  number  = {4},
  pages   = {1533--1574},
  doi     = {10.3934/dcds.2014.34.1533}
}

@article{haasler2021multimarginal,
  author  = {Haasler, Isabel and Ringh, Axel
             and Chen, Yongxin and Karlsson, Johan},
  title   = {Multimarginal Optimal Transport with a
             Tree-Structured Cost and the
             {Schr{\"o}dinger} Bridge Problem},
  journal = {SIAM Journal on Control and Optimization},
  year    = {2021},
  volume  = {59},
  number  = {4},
  pages   = {2428--2453},
  doi     = {10.1137/20M1320195}
}

@inproceedings{noble2023tree,
  author    = {Noble, Maxence and De Bortoli, Valentin
               and Doucet, Arnaud and Durmus, Alain},
  title     = {Tree-Based Diffusion {Schr{\"o}dinger} Bridge
               with Applications to {Wasserstein} Barycenters},
  booktitle = {Advances in Neural Information Processing Systems},
  volume    = {36},
  year      = {2023},
  url       = {https://proceedings.neurips.cc/paper_files/paper/2023/hash/ad08767706825033b99122332293033d-Abstract-Conference.html}
}

@inproceedings{howard2025tree,
  author    = {Howard, Samuel and Potaptchik, Peter
               and Deligiannidis, George},
  title     = {{Schr{\"o}dinger} Bridge Matching for
               Tree-Structured Costs and Entropic
               {Wasserstein} Barycentres},
  booktitle = {Advances in Neural Information Processing Systems},
  volume    = {38},
  year      = {2025},
  url       = {https://openreview.net/forum?id=DliPKnn6e0}
}

@article{bass1983skorokhod,
  author  = {Bass, Richard F.},
  title   = {{Skorokhod} Imbedding via Stochastic Integrals},
  journal = {S{\'e}minaire de Probabilit{\'e}s},
  year    = {1983},
  volume  = {17},
  pages   = {221--224},
  url     = {https://www.numdam.org/item/SPS_1983__17__221_0/}
}

@article{backhoff2020martingale,
  author  = {Backhoff-Veraguas, Julio and Beiglb{\"o}ck, Mathias
             and Huesmann, Martin and K{\"a}llblad, Sigrid},
  title   = {Martingale {Benamou--Brenier}:
             A Probabilistic Perspective},
  journal = {The Annals of Probability},
  year    = {2020},
  volume  = {48},
  number  = {5},
  pages   = {2258--2289},
  doi     = {10.1214/20-AOP1422}
}

@article{gozlan2017kantorovich,
  author  = {Gozlan, Natha{\"e}l and Roberto, Cyril
             and Samson, Paul-Marie and Tetali, Prasad},
  title   = {{Kantorovich} Duality for General Transport Costs
             and Applications},
  journal = {Journal of Functional Analysis},
  year    = {2017},
  volume  = {273},
  number  = {11},
  pages   = {3327--3405},
  doi     = {10.1016/j.jfa.2017.08.015}
}

@article{backhoff2019existence,
  author  = {Backhoff-Veraguas, Julio and Beiglb{\"o}ck, Mathias
             and Pammer, Gudmund},
  title   = {Existence, Duality, and Cyclical Monotonicity
             for Weak Transport Costs},
  journal = {Calculus of Variations and Partial Differential
             Equations},
  year    = {2019},
  volume  = {58},
  number  = {6},
  pages   = {203},
  doi     = {10.1007/s00526-019-1624-y}
}

@article{tan2013optimal,
  author  = {Tan, Xiaolu and Touzi, Nizar},
  title   = {Optimal Transportation under Controlled
             Stochastic Dynamics},
  journal = {The Annals of Probability},
  year    = {2013},
  volume  = {41},
  number  = {5},
  pages   = {3201--3240},
  doi     = {10.1214/12-AOP797}
}

@misc{alouadi2026pde,
  author        = {Alouadi, Alexandre
                   and Henry-Labord{\`e}re, Pierre
                   and Loeper, Gr{\'e}goire
                   and Mazhar, Othmane
                   and Pham, Huy{\^e}n and Touzi, Nizar},
  title         = {A {PDE} Derivation of the
                   {Schr{\"o}dinger--Bass} Bridge},
  year          = {2026},
  eprint        = {2601.17863},
  archivePrefix = {arXiv},
  url           = {https://arxiv.org/abs/2601.17863v1}
}

@misc{henry2026bridging,
  author        = {Henry-Labord{\`e}re, Pierre
                   and Loeper, Gr{\'e}goire
                   and Mazhar, Othmane
                   and Pham, Huy{\^e}n and Touzi, Nizar},
  title         = {Bridging {Schr{\"o}dinger} and {Bass}:
                   A Semimartingale Optimal Transport Problem
                   with Diffusion Control},
  year          = {2026},
  eprint        = {2603.27712},
  archivePrefix = {arXiv},
  note          = {Version 2, May 29, 2026},
  url           = {https://arxiv.org/abs/2603.27712v2}
}

@misc{hasenbichler2026weak,
  author        = {Hasenbichler, Manuel and Pammer, Gudmund
                   and Thonhauser, Stefan},
  title         = {A Weak Transport Approach to the
                   {Schr{\"o}dinger--Bass} Bridge},
  year          = {2026},
  eprint        = {2604.02312},
  archivePrefix = {arXiv},
  note          = {Version 2, June 30, 2026},
  url           = {https://arxiv.org/abs/2604.02312v2}
}

@misc{alouadi2026lightsbb,
  author        = {Alouadi, Alexandre
                   and Henry-Labord{\`e}re, Pierre
                   and Loeper, Gr{\'e}goire
                   and Mazhar, Othmane
                   and Pham, Huy{\^e}n and Touzi, Nizar},
  title         = {{LightSBB-M}: Bridging {Schr{\"o}dinger}
                   and {Bass} for Generative Diffusion Modeling},
  year          = {2026},
  eprint        = {2601.19312},
  archivePrefix = {arXiv},
  note          = {Version 2, May 5, 2026},
  url           = {https://arxiv.org/abs/2601.19312v2}
}

@article{del2020statistical,
  title={The statistical effect of entropic regularization in optimal transportation},
  author={del Barrio, Eustasio and Loubes, Jean-Michel},
  journal={arXiv preprint arXiv:2006.05199},
  year={2020}
}

@article{mallasto2022entropy,
  author  = {Mallasto, Anton and Gerolin, Augusto
             and Minh, H{\`a} Quang},
  title   = {Entropy-Regularized 2-{Wasserstein} Distance
             between {Gaussian} Measures},
  journal = {Information Geometry},
  year    = {2022},
  volume  = {5},
  pages   = {289--323},
  doi     = {10.1007/s41884-021-00052-8}
}

@inproceedings{janati2020gaussian,
  author    = {Janati, Hicham and Muzellec, Boris
               and Peyr{\'e}, Gabriel and Cuturi, Marco},
  title     = {Entropic Optimal Transport between Unbalanced
               {Gaussian} Measures Has a Closed Form},
  booktitle = {Advances in Neural Information Processing Systems},
  volume    = {33},
  year      = {2020},
  url       = {https://proceedings.neurips.cc/paper/2020/hash/766e428d1e232bbdd58664b41346196c-Abstract.html}
}

@inproceedings{bunne2023gaussian,
  author    = {Bunne, Charlotte and Hsieh, Ya-Ping
               and Cuturi, Marco and Krause, Andreas},
  title     = {The {Schr{\"o}dinger} Bridge between
               {Gaussian} Measures Has a Closed Form},
  booktitle = {Proceedings of the 26th International Conference
               on Artificial Intelligence and Statistics},
  series    = {Proceedings of Machine Learning Research},
  volume    = {206},
  pages     = {5802--5833},
  year      = {2023},
  url       = {https://proceedings.mlr.press/v206/bunne23a.html}
}

@article{alvarez2016fixedpoint,
  author  = {{\'A}lvarez-Esteban, Pedro C.
             and del Barrio, Eustasio
             and Cuesta-Albertos, J. A.
             and Matr{\'a}n, Carlos},
  title   = {A Fixed-Point Approach to Barycenters in
             {Wasserstein} Space},
  journal = {Journal of Mathematical Analysis and Applications},
  year    = {2016},
  volume  = {441},
  number  = {2},
  pages   = {744--762},
  doi     = {10.1016/j.jmaa.2016.04.045}
}

@inproceedings{altschuler2021averaging,
  author    = {Altschuler, Jason M. and Chewi, Sinho
               and Gerber, Patrik R. and Stromme, Austin J.},
  title     = {Averaging on the {Bures--Wasserstein} Manifold:
               Dimension-Free Convergence of Gradient Descent},
  booktitle = {Advances in Neural Information Processing Systems},
  volume    = {34},
  pages     = {22132--22145},
  year      = {2021},
  url       = {https://proceedings.neurips.cc/paper_files/paper/2021/hash/b9acb4ae6121c941324b2b1d3fac5c30-Abstract.html}
}

@inproceedings{nguyenbarycenter,
  title={On Barycenter Computation: Analyzing Semi-Unbalanced Optimal Transport-based Method on Bures-Wasserstein manifold.},
  author={Nguyen, Ngoc-Hai and Nguyen, Hoang-Phi and Pham, Tung and Ho, Nhat and others},
  year      = {2025},
  booktitle={The 29th International Conference on Artificial Intelligence and Statistics}
}

@article{schiebinger2019optimal,
  title={Optimal-transport analysis of single-cell gene expression identifies developmental trajectories in reprogramming},
  author={Schiebinger, Geoffrey and Shu, Jian and Tabaka, Marcin and Cleary, Brian and Subramanian, Vidya and Solomon, Aryeh and Gould, Joshua and Liu, Siyan and Lin, Stacie and Berube, Peter and others},
  journal={Cell},
  volume={176},
  number={4},
  pages={928--943},
  year={2019},
  publisher={Elsevier}
}

@article{kester2018single,
  title={Single-cell transcriptomics meets lineage tracing},
  author={Kester, Lennart and Van Oudenaarden, Alexander},
  journal={Cell stem cell},
  volume={23},
  number={2},
  pages={166--179},
  year={2018},
  publisher={Elsevier}
}

@inproceedings{werenski2022measure,
  title={Measure estimation in the barycentric coding model},
  author={Werenski, Matthew and Jiang, Ruijie and Tasissa, Abiy and Aeron, Shuchin and Murphy, James M},
  booktitle={International Conference on Machine Learning},
  pages={23781--23803},
  year={2022},
  organization={PMLR}
}

@article{cheng2021dynamical,
  title={Dynamical Wasserstein barycenters for time-series modeling},
  author={Cheng, Kevin and Aeron, Shuchin and Hughes, Michael C and Miller, Eric L},
  journal={Advances in Neural Information Processing Systems},
  volume={34},
  pages={27991--28003},
  year={2021}
}

@article{mallery2025synthesis,
  title={Synthesis and analysis of data as probability measures with entropy-regularized optimal transport},
  author={Mallery, Brendan and Murphy, James M and Aeron, Shuchin},
  journal={arXiv preprint arXiv:2501.07446},
  year={2025}
}

@inproceedings{
werenski2025linearized,
title={Linearized Wasserstein Barycenters: Synthesis, Analysis, Representational Capacity, and Applications},
author={Matthew Werenski and Brendan Mallery and Shuchin Aeron and James M. Murphy},
booktitle={The 28th International Conference on Artificial Intelligence and Statistics},
year={2025},
url={https://openreview.net/forum?id=MWCHedOm2L}
}

@article{JMLR:v25:23-0708,
  author  = {Florian Gunsilius and Meng Hsuan Hsieh and Myung Jin Lee},
  title   = {Tangential Wasserstein Projections},
  journal = {Journal of Machine Learning Research},
  year    = {2024},
  volume  = {25},
  number  = {69},
  pages   = {1--41},
  url     = {http://jmlr.org/papers/v25/23-0708.html}
}

@inproceedings{bai2023linear,
  title={Linear optimal partial transport embedding},
  author={Bai, Yikun and Medri, Ivan Vladimir and Martin, Rocio Diaz and Shahroz, Rana and Kolouri, Soheil},
  booktitle={International Conference on Machine Learning},
  pages={1492--1520},
  year={2023},
  organization={PMLR}
}

@article{schmitz2018wasserstein,
  title={Wasserstein dictionary learning: Optimal transport-based unsupervised nonlinear dictionary learning},
  author={Schmitz, Morgan A and Heitz, Matthieu and Bonneel, Nicolas and Ngole, Fred and Coeurjolly, David and Cuturi, Marco and Peyr{\'e}, Gabriel and Starck, Jean-Luc},
  journal={SIAM Journal on Imaging Sciences},
  volume={11},
  number={1},
  pages={643--678},
  year={2018},
  publisher={SIAM}
}

@article{xu2018distilled,
  title={Distilled wasserstein learning for word embedding and topic modeling},
  author={Xu, Hongteng and Wang, Wenlin and Liu, Wei and Carin, Lawrence},
  journal={Advances in Neural Information Processing Systems},
  volume={31},
  year={2018}
}

@inproceedings{tong2020trajectorynet,
  title={Trajectorynet: A dynamic optimal transport network for modeling cellular dynamics},
  author={Tong, Alexander and Huang, Jessie and Wolf, Guy and Van Dijk, David and Krishnaswamy, Smita},
  booktitle={International conference on machine learning},
  pages={9526--9536},
  year={2020},
  organization={PMLR}
}

@article{nilsson2020understanding,
  title={Understanding ssim},
  author={Nilsson, Jim and Akenine-M{\"o}ller, Tomas},
  journal={arXiv preprint arXiv:2006.13846},
  year={2020}
}
\bibliographystyle{plainnat}

\clearpage

\appendix
\thispagestyle{empty}
\tableofcontents  
\addtocontents{toc}{\protect\setcounter{tocdepth}{1}}
\numberwithin{equation}{section}
\numberwithin{theorem}{section}
\numberwithin{algorithm}{section}
\numberwithin{figure}{section}
\numberwithin{table}{section}

\section{MORE BACKGROUND ON SCHR\"ODINGER BASS BRIDGE AND PROPERTIES}
\label{app:background}
The table below summarizes the main differences in intuition and dynamic formulation among the Schrödinger Bridge, Bass Martingale, and SBB frameworks. We also note that the feasibility of the Bass martingale problem is characterized by the convex-order condition, whereas no such constraint is required in the SBB formulation.
\begin{center}
\scriptsize
\setlength{\tabcolsep}{3pt}
\renewcommand{\arraystretch}{1.1}
\begin{tabular}{llll}
\toprule
 & \textbf{Schr\"odinger Bridge}~\cite{schrodinger1932theorie}
 & \textbf{Bass Martingale}~\cite{conze2021bass}
 & \textbf{SBB}~\cite{henry2026bridging} \\
\midrule

Drift
& Controlled
& Zero
& Controlled \\

Volatility
& Fixed
& Controlled
& Controlled \\

Dynamics
& $dX_t=\alpha_tdt+\sqrt{\varepsilon}\,dW_t$
& $dX_t=\sigma_t\,dW_t$
& $dX_t=\alpha_tdt+\sigma_t\,dW_t$ \\

Main constraint
& Endpoint marginals
& Endpoint marginals + martingale
& Endpoint marginals \\

\bottomrule
\end{tabular}
\end{center}

We present some useful properties of SBB which might be helpful in analysis this problem

\begin{itemize}
        \item Translation invariance: 
For $a \in \mathbb{R}^d$, let $\tau_a(x)=x+a$. Then

$$
\mathsf{SBB}\left(\tau_a \# \mu, \tau_a \# \rho\right)=\mathsf{SBB}(\mu, \rho) .
$$  
    \begin{proof}
        Take any admissible process $\boldsymbol{X}_{\boldsymbol{t}}$ from $\mu$ to $\rho$
    $$
    X_0 \sim \mu, \quad X_T \sim \rho,
    $$
    and
    $$
    d X_t=\alpha_t d t+\sigma_t d W_t .
    $$
    Define the translated process
    $$
    Y_t=X_t+a .
    $$
    Then
    $$
    Y_0=X_0+a \sim \tau_{a \# \mu}, \quad Y_T=X_T+a \sim \tau_{a \#} \rho .
    $$
    Also,
    $$
    d Y_t=d X_t=\alpha_t d t+\sigma_t d W_t .
    $$
    So the drift and volatility are exactly the same. Therefore the cost is unchanged.
    \end{proof}
        
        \item Orthogonal invariance: 
For an orthogonal matrix $Q$,
$$
\mathsf{SBB}(Q \# \mu, Q \# \rho)=\mathsf{SBB}(\mu, \rho) .
$$
    \begin{proof}
        Take any admissible process
$$
d X_t=\alpha_t d t+\sigma_t d W_t, \quad X_0 \sim \mu, \quad X_T \sim \rho .
$$
Define
$$
Y_t=Q X_t .
$$
Then
$$
Y_0=Q X_0 \sim Q_{\# \mu}, \quad Y_T=Q X_T \sim Q_{\#} \rho .
$$
Now compute the dynamics:
$$
d Y_t=Q d X_t=Q \alpha_t d t+Q \sigma_t d W_t .
$$
Since $Q$ is orthogonal, the process
$$
B_t=Q W_t
$$
is again a standard Brownian motion. Equivalently,
$$
d W_t=Q^{\top} d B_t .
$$
Thus
$$
d Y_t=Q \alpha_t d t+Q \sigma_t Q^{\top} d B_t .
$$
So the transformed drift and volatility are
$$
\alpha_t^Q=Q \alpha_t, \quad \sigma_t^Q=Q \sigma_t Q^{\top} .
$$
Now compare the cost. For the drift part,
$$
\left\|\alpha_t^Q\right\|^2=\left\|Q \alpha_t\right\|^2=\alpha_t^{\top} Q^{\top} Q \alpha_t=\left\|\alpha_t\right\|^2 .
$$

For the volatility part,

$$
Q \sigma_t Q^{\top}-\sqrt{\varepsilon} I_d=Q \sigma_t Q^{\top}-\sqrt{\varepsilon} Q I_d Q^{\top}=Q\left(\sigma_t-\sqrt{\varepsilon} I_d\right) Q^{\top} .
$$

The Frobenius norm is invariant under orthogonal transformations, so

$$
\left\|Q\left(\sigma_t-\sqrt{\varepsilon} I_d\right) Q^{\top}\right\|_F=\left\|\sigma_t-\sqrt{\varepsilon} I_d\right\|_F .
$$

Therefore

$$
J\left(Y, \alpha^Q, \sigma^Q\right)=J(X, \alpha, \sigma) .
$$
    \end{proof}

        \item Not affine invariant because the reference volatility $\sqrt{\varepsilon} I_d$ fixes an isotropic Euclidean scale.
        \item Directed heat-flow bias
$$
\mathsf{SBB}\left(\mu, H_T \mu\right)=0, \quad H_T \mu=\mu * \mathcal{N}\left(0, \varepsilon T I_d\right)
$$
        \begin{proof}
            Indeed, let $$
\alpha_t=0, \quad \sigma_t=\sqrt{\varepsilon} I_d .
$$
Therefore the cost integrand is
$$
\frac{1}{2}\left\|\alpha_t\right\|^2+\frac{\beta}{2}\left\|\sigma_t-\sqrt{\varepsilon} I_d\right\|_F^2=0
$$
Hence
$$
J\left(X, 0, \sqrt{\varepsilon} I_d\right)=0 .
$$
Also,
$$
X_T=X_0+\sqrt{\varepsilon} W_T .
$$
Since
$$
\sqrt{\varepsilon} W_T \sim N\left(0, \varepsilon T I_d\right),
$$
we get
$$
X_T \sim \mu * N\left(0, \varepsilon T I_d\right)=H_T \mu .
$$
Thus there exists an admissible process from $\mu$ to $H_T \mu$ with zero cost.
        \end{proof}

    \end{itemize}


\section{MATHEMATICAL PROOFS FOR THE MAIN RESULTS} \label{app:main_proof}

\subsection{Proof for the existence of the decoder SBB barycenter (\Cref{thm:existence-main})} \label{proof:existence}
\textbf{Proof.}
First, the functional is finite somewhere. For example, fix a candidate $\rho\in\Ptwo(\mathbb{R}^d)$. For each $i$, choose a quadratic optimal coupling $(U,V)$ of $\rho$ and $\mu_i$. Consider
\[
X_t
=
\left(1-\frac{t}{T}\right)U
+
\frac{t}{T}V.
\]
Then $X_0\sim\rho$, $X_T\sim\mu_i$, and
\[
\alpha_t=\frac{V-U}{T},
\qquad
\sigma_t=0.
\]
Hence
\[
\begin{aligned}
J(P)
&=
\frac{1}{2T}\mathbb E\|V-U\|^2
+
\frac{\beta\varepsilon dT}{2} \\
&=
\frac{W_2^2(\rho,\mu_i)}{2T}
+
\frac{\beta\varepsilon dT}{2}.
\end{aligned}
\]
Therefore
\[
0 \leq
\SBB_{\beta,\varepsilon,T}(\rho,\mu_i)
\leq \frac{W_2^2(\rho,\mu_i)}{2T} + \frac{\beta\varepsilon dT}{2}
<\infty.
\]
Then the infimum exists and there is a minimizing sequence $\{\sum_i\lambda_i\mathsf{SBB}_{\beta,\varepsilon,T}(\rho,\mu_{i_n})\}_n$ converging to this value. 

Now take any admissible process from $\rho$ to $\mu_i$,
\[
X_T-X_0=\int_0^T\alpha_t\dd t+\int_0^T\sigma_t\dd W_t,
\]
We have
\[
W_2^2(\rho,\mu_i)\leq \E\|X_T-X_0\|^2.
\]
Using Cauchy-Schwarz,
$$
\mathbb{E}\left\|\int_0^T \alpha_t d t\right\|^2 \leq T \mathbb{E} \int_0^T\left\|\alpha_t\right\|^2 d t .
$$
Using Ito isometry (details in \ref{proof:ito}),
$$
\mathbb{E}\left\|\int_0^T \sigma_t d W_t\right\|^2=\mathbb{E} \int_0^T\left\|\sigma_t\right\|_F^2 d t
$$
So
$$
\mathbb{E}\left\|X_T-X_0\right\|^2 \leq 2 T \mathbb{E} \int_0^T\left\|\alpha_t\right\|^2 d t+2 \mathbb{E} \int_0^T\left\|\sigma_t\right\|_F^2 d t
$$
Also, by Cauchy-Schwarz
\[
\|\sigma_t\|_F^2\leq 2\|\sigma_t-\sqrt{\varepsilon}\Id\|_F^2+2\varepsilon d.
\]
Therefore, for
\[
J(P)=\E_{\mathbb{P}}\int_0^T\left[\frac12\|\alpha_t\|^2+\frac{\beta}{2}\|\sigma_t-\sqrt{\varepsilon}\Id\|_F^2\right]\dd t,
\]
we obtain
\[
W_2^2(\rho,\mu_i)
\leq \left(4T+\frac8\beta\right)J(P)+4\varepsilon dT.
\]
Indeed, for the drift part
$$\mathbb E\int_0^T\|\alpha_t\|^2dt
\le
2J(P)$$
For the diffusion part, 
$$
\mathbb{E} \int_0^T\left\|\sigma_t\right\|_F^2 d t \leq 2 \mathbb{E} \int_0^T\left\|\sigma_t-\sqrt{\varepsilon} I_d\right\|_F^2 d t+2 \varepsilon d T .
$$
And since
$$
\frac{\beta}{2} \mathbb{E} \int_0^T\left\|\sigma_t-\sqrt{\varepsilon} I_d\right\|_F^2 d t \leq J(P),
$$
we have
$$
\mathbb{E} \int_0^T\left\|\sigma_t-\sqrt{\varepsilon} I_d\right\|_F^2 d t \leq \frac{2}{\beta} J(P) .
$$
So
$$
\mathbb{E} \int_0^T\left\|\sigma_t\right\|_F^2 d t \leq \frac{4}{\beta} J(P)+2 \varepsilon d T
$$
Plugging back:

$$
\mathbb{E}\left\|X_T-X_0\right\|^2 \leq 2 T(2 J(P))+2\left(\frac{4}{\beta} J(P)+2 \varepsilon d T\right) .
$$

So
$$
\mathbb{E}\left\|X_T-X_0\right\|^2 \leq\left(4 T+\frac{8}{\beta}\right) J(P)+4 \varepsilon d T .
$$
Therefore,
$$
W_2^2\left(\mu_i, \rho\right) \leq\left(4 T+\frac{8}{\beta}\right) J(P)+4 \varepsilon d T.
$$
Taking the infimum over $P$,
\[
\mathsf{SBB}_{\beta,\varepsilon,T}(\rho,\mu_i)
\geq
\frac{W_2^2(\rho,\mu_i)-4\varepsilon dT}{4T+8/\beta}.
\]
Thus, as $SBB(\rho_n,\mu_i)$ is upper bounded, then $W_2^2(\rho_n,\mu_i)\le C_i'$ for a constant $C_i'$ that is independent of $n$; hence every minimizing sequence $(\rho_n)$ has uniformly bounded $W_2(\rho_n,\mu_i)$.  Then it has uniformly bounded second moments. Indeed, let $X_n \sim \rho_n$, we want to prove that $\mathbb{E}\|X_n\|^2$ is bounded. To do this, let $X_0 \sim \mu_i$, as $\mu_i \in \Ptwo(\R^d)$, $\mathbb{E}\|X_0\|^2$ is bounded. We use the inequality
$$
\|X_n\|^2=\|X_n-X_0+X_0\|^2 \leq 2\|X_n-X_0\|^2+2\|X_0\|^2 .
$$
Take the expectation
$$
\mathbb{E}\|X_n\|^2 \leq 2 \mathbb{E}\|X_n-X_0\|^2+2 \mathbb{E}\|X_0\|^2 .
$$
As the inequality holds for all coupling of $(X_0, X_n)$, we have 
$$
\mathbb{E}\|X_n\|^2 \leq 2 W_2^2(\rho_n, \mu_i)+2 \mathbb{E}\|X_0\|^2
$$
Therefore the minimizing sequence $\{\rho_n\}$ is  uniformly tight in $\Ptwo(\mathbb{R}^d)$. Then by theorem of Prokhorov \cite{billingsley2013convergence}, there exists a subsequence $\{\rho'_n\} \subset \{\rho_n\}$ that converges weakly to a limit $\rho' \in \Ptwo(\mathbb{R}^d)$. 

By the weak lower semicontinuity Lemma \ref{lem:lower_semicontinuity}, for every $i$,
\[
\SBB_{\beta,\varepsilon,T}(\rho',\mu_i)
\le
\liminf_{n\to\infty}
\SBB_{\beta,\varepsilon,T}(\rho_n,\mu_i).
\]
Therefore
\[
\begin{aligned}
\mathcal{B}_{\mathrm{dec}}(\rho')
&=
\sum_{i=1}^N
\lambda_i\SBB_{\beta,\varepsilon,T}(\rho',\mu_i) \\
&\le
\sum_{i=1}^N
\lambda_i
\liminf_{n\to\infty}
\SBB_{\beta,\varepsilon,T}(\rho_n,\mu_i) \\
&\le
\liminf_{n\to\infty}
\mathcal{B}_{\mathrm{dec}}(\rho_n) \\
&=
b_*.
\end{aligned}
\]
By definition of $b_*$, the reverse inequality is automatic. Thus
\[
\mathcal{B}_{\mathrm{dec}}(\rho')=b_*,
\]
and therefore $\rho'
\in
\operatorname*{arg\,min}_{\rho\in\mathcal{P}_2(\R^d)}
\sum_{i=1}^N
\lambda_i
\SBB_{\beta,\varepsilon,T}(\rho,\mu_i).$
\qedbox

\paragraph{Remark.} The proof above can be straightforwardly adapted to the encoder SBB barycenter case $\rho\mapsto\sum_i\lambda_i\mathsf{SBB}_{\beta,\varepsilon,T}(\mu_i,\rho)$. Indeed, from the beginning of the argument up to the extraction of a weakly convergent subsequence $\{\rho'_n\} \rightharpoonup \rho'$ step, all the techniques employed are independent of the ordering of $\rho$ and $\mu_i$ in SBB functional. The final step is even simpler, thanks to continuity of the SBB for the second marginal according to $W_2$ proved in \cite{henry2026bridging} (and also metrizes the weak convergence), which implies that $\rho'$ achieves the infimum. \\

Meanwhile, for the symmetric and Markovian cases, the same proof strategy applies to $\mathcal{F}(\rho)=\sum_{i=1}^n\left[a_i \operatorname{SBB}\left(\mu_i, \rho\right)+b_i \operatorname{SBB}\left(\rho, \mu_i\right)\right]$. For the symmetric objective, take $a_i=b_i=\lambda_i / 2$. For the Markovian objective, take $a_1=\lambda_1, b_2=\lambda_2$, and all other coefficients zero. Since the Wasserstein bound holds in either orientation, combining it with the joint weak lower semicontinuity of SBB, established from the encoder and decoder cases above, yields the conclusion.


\paragraph{Proof of Ito isometry step.} \label{proof:ito}
Consider
$$
M=\int_0^T \sigma_t d W_t
$$
Here $W_t$ is a $d$-dimensional Brownian motion, and $\sigma_t$ is a matrix-valued process. The stochastic integral is vector-valued. The Ito isometry says:
$$
\mathbb{E}\left\|\int_0^T \sigma_t d W_t\right\|^2=\mathbb{E} \int_0^T\left\|\sigma_t\right\|_F^2 d t .
$$
Here $\left\|\sigma_t\right\|_F$ is the Frobenius norm:
$$
\left\|\sigma_t\right\|_F^2=\sum_{i=1}^d \sum_{j=1}^d\left(\sigma_t^{i j}\right)^2 .
$$
To see why, write the $i$-th coordinate of $M$ :
$$
M^i=\sum_{j=1}^d \int_0^T \sigma_t^{i j} d W_t^j .
$$
Then
$$
\|M\|^2=\sum_{i=1}^d\left(M^i\right)^2
$$
By Itô isometry and independence/orthogonality of Brownian coordinates,
$$
\mathbb{E}[\left(M^i\right)^2]=\mathbb{E} \int_0^T \sum_{j=1}^d\left(\sigma_t^{i j}\right)^2 d t
$$
Summing over $i$,

$$
\mathbb{E}\|M\|^2=\sum_{i=1}^d \mathbb{E}[\left(M^i\right)^2]=\mathbb{E} \int_0^T \sum_{i=1}^d \sum_{j=1}^d\left(\sigma_t^{i j}\right)^2 d t .
$$

Therefore,
$$
\mathbb{E}\left\|\int_0^T \sigma_t d W_t\right\|^2=\mathbb{E} \int_0^T\left\|\sigma_t\right\|_F^2 d t .
$$
\qedbox

\paragraph{Weak lower semicontinuity in the source marginal}
\begin{lemma}[Weak lower semicontinuity in the source marginal]\label{lem:lower_semicontinuity}
Fix $\nu\in\mathcal{P}_2(\R^d)$. If
\[
\rho_n\rightharpoonup\rho
\]
weakly and $\rho_n,\rho\in\mathcal{P}_2(\R^d)$, then
\[
\SBB_{\beta,\varepsilon,T}(\rho,\nu)
\le
\liminf_{n\to\infty}
\SBB_{\beta,\varepsilon,T}(\rho_n,\nu).
\]
\end{lemma}

\begin{proof}
Define $\mathcal{P}_x=\left\{\mathbb{P} \in \mathcal{P}: X_0=x \quad \text { $\mathbb{P}$-a.s., } J(\mathbb{P})<\infty\right\}$. Let $\Psi$ be a class of continuous terminal potentials $\psi$ that are bounded above and have at most quadratic growth. By the dual formulation of the SBB problem,
\[
\SBB_{\beta,\varepsilon,T}(\rho,\nu)
=
\sup_{\psi\in\Psi}
\left\{
\int_{\R^d}\psi(y)\,\nu(dy)
-
\int_{\R^d}v^\psi(0,x)\,\rho(dx)
\right\}.
\]
with 
\begin{equation*}
    v^\psi(0,x)
:=
\sup_{X_0=x}
\left\{
\mathbb E[\psi(X_T)]-J(\mathbb{P})
\right\},
\end{equation*}

We first show that $v^\psi(0, \cdot)$ is finite, continuous, and bounded above. Since $J(\mathbb{P})\ge0$,
\begin{equation} \label{eq:dual_relation}
    v^\psi(0,x)\le \sup_{y\in\R^d}\psi(y).
\end{equation}
On the other hand, the deterministic path
\[
X_t=\left(1-\frac{t}{T}\right)x
\]
satisfies $X_0=x$ and $X_T=0$, with drift $-x/T$ and volatility $0$. Therefore
\[
v^\psi(0,x)
\ge
\psi(0)
-\frac{\|x\|^2}{2T}
-\frac{\beta\varepsilon dT}{2}.
\]
Hence $v^\psi(0, \cdot)$ is finite everywhere. Now fix $\mathbb{Q}\in\mathcal P_0$ and define, under $\mathbb{Q}$,
\[
X_t^x:=X_t+\left(1-\frac{t}{T}\right)x,
\qquad
\mathbb{P}^x:=\operatorname{Law}_{\mathbb{Q}}(X^x).
\]
Then $\mathbb{P}^x\in \mathcal{P}_x$, and
\[
J(\mathbb{P}^x)
=
J(\mathbb{Q})-\frac{x}{T} \mathbb E_{\mathbb{Q}}[X_T]
+\frac{\|x\|^2}{2T}.
\]
Since this translation is invertible, we obtain
\[
v^\psi(0,x)+\frac{\|x\|^2}{2T}
=
\sup_{\mathbb{Q}\in\mathcal P}
\left\{
\mathbb E_{\mathbb{Q}}[\psi(X_T)]-J(\mathbb{Q})
+\frac{x}{T} \mathbb E_{\mathbb{Q}}[X_T]
\right\}.
\]

The right-hand side is a supremum of affine functions of $x$, hence is convex. Since it is finite on all of $\R^d$, it is continuous. Therefore $v^\psi(0, \cdot)$ is continuous. Moreover, by \ref{eq:dual_relation}, $-v^\psi$ is continuous and bounded below. Hence, by the Portmanteau theorem,
\[
\int_{\R^d}-v^\psi(0,x)\,\rho(dx)
\le
\liminf_{n\to\infty}
\int_{\R^d}-v^\psi(0,x)\,\rho_n(dx).
\]
Therefore, for every $\psi\in\Psi$,
\[
\begin{aligned}
\int\psi\,d\nu-\int v^\psi\,d\rho
&\le
\liminf_{n\to\infty}
\left(
\int\psi\,d\nu-\int v^\psi\,d\rho_n
\right) \\
&\le
\liminf_{n\to\infty}
\SBB_{\beta,\varepsilon,T}(\rho_n,\nu).
\end{aligned}
\]
Taking the supremum over $\psi\in\Psi$ proves the claim.
\end{proof}

\subsection{Proof for the Gaussian geometry of the decoder SBB Barycenter}
\label{app:main_proof_gaussian}
\subsubsection{Proof for Proposition \ref{prop:triple-static}}
\begin{proof}
The work of \cite{hasenbichler2026weak} gives the pointwise static SBB cost as
\begin{equation*}
C_{\beta,\varepsilon,T}(x,\eta)
=
\inf_{\substack{
\kappa\in\mathcal P_2(\mathbb R^d)\\
\overline\kappa=\overline\eta
}}
\left\{
\varepsilon
H(\kappa\mid\gamma_{x,\tau})
+
\frac{\beta}{2}
W_2^2(\kappa,\eta)
\right\},
\label{eq:point-static}
\end{equation*}
and the global weak-transport formulation is
\begin{equation*}
\SBB_{\beta,\varepsilon,T}(\mu,\nu)
=
\inf_{\pi\in\Pi(\mu,\nu)}
\int
C_{\beta,\varepsilon,T}(x,\pi_x)\,\mu(dx).
\label{eq:global-static}
\end{equation*}

Let \(\Gamma\in\mathcal A(\mu,\nu)\), and disintegrate it as
\[
\Gamma(dx\,dy\,dz)
=
\mu(dx)\Gamma_x(dy\,dz).
\]
Set
\[
\kappa_x
:=
(\operatorname{pr}_Y)_\#\Gamma_x,
\qquad
\eta_x
:=
(\operatorname{pr}_Z)_\#\Gamma_x.
\]
Because of the conditional-mean constraint,
\[
\overline{\kappa_x} = \overline{\eta_x}
\quad
\mu\text{-a.e.} \qquad \text{ i.e. } \int_{\mathbb{R}^d} y \kappa_x(d y)=\int_{\mathbb{R}^d} z \eta_x(d z) \quad \mu \text {-a.e. }
\]
Chain rule for relative entropy gives
\begin{align*}
H\!\left(
\Gamma_{XY}
\,\middle|\,
\mu(dx)\gamma_{x,\tau}(dy)
\right)
& = H(\Gamma_X \mid \mu) + \int
H((\operatorname{pr}_Y)_\#\Gamma_x \mid\gamma_{x,\tau})\,\Gamma_X(dx) \\
& =\int H(\kappa_x\mid\gamma_{x,\tau})\,\mu(dx)\quad \text{ (as $\Gamma_X \ = \mu$).}
\label{eq:entropy-disintegration-static}
\end{align*}

Moreover, since \(\Gamma_x\) is a coupling of \(\kappa_x\) and \(\eta_x\),
\begin{equation*}
\int\|y-z\|^2\,\Gamma_x(dy\,dz)
\geq
W_2^2(\kappa_x,\eta_x).
\label{eq:conditional-W2}
\end{equation*}
It follows that $\mathcal J_{\beta,\varepsilon,T}(\Gamma) \geq C_{\beta,\varepsilon,T}(x,\eta_x) \quad \forall x$. Then 
$$
\mathcal J_{\beta,\varepsilon,T}(\Gamma)
\geq
\int
C_{\beta,\varepsilon,T}(x,\eta_x)\,\mu(dx)
\geq
\SBB_{\beta,\varepsilon,T}(\mu,\nu).
$$
Taking the infimum over \(\Gamma\) yields
\[
\inf_{\Gamma\in\mathcal A(\mu,\nu)}
\mathcal J_{\beta,\varepsilon,T}(\Gamma)
\geq
\SBB_{\beta,\varepsilon,T}(\mu,\nu).
\]
For the last inequality, we note that
\[
\pi(dx\,dz)
:=
\mu(dx)\eta_x(dz)
\]
belongs to \(\Pi(\mu,\nu)\), hence is feasible for the infimum defining.

Conversely, fix \(\delta>0\). Choose an endpoint coupling
\[
\pi(dx\,dz)=\mu(dx)\pi_x(dz)\in\Pi(\mu,\nu)
\]
and a measurable kernel \(x\mapsto\kappa_x\) such that
\[
\overline{\kappa_x}
=
\overline{\pi_x}
\qquad
\mu\text{-a.e.}
\]
and
\begin{align*}
\int
\Bigg[\varepsilon H(\kappa_x\mid\gamma_{x,\tau})+ \frac{\beta}{2}W_2^2(\kappa_x,\pi_x)\Bigg]\mu(dx) \leq \SBB_{\beta,\varepsilon,T}(\mu,\nu)+\delta.
\label{eq:delta-selector}
\end{align*}

For \(\mu\)-almost every \(x\), choose an optimal quadratic coupling
\[
\lambda_x\in\Pi(\kappa_x,\pi_x).
\]
Such a family can be chosen measurably because the underlying spaces are Polish and the quadratic cost is continuous; alternatively one may use universally measurable \(\delta\)-optimal selectors.

Define
\[
\Gamma(dx\,dy\,dz)
:=
\mu(dx)\lambda_x(dy\,dz).
\]
Then
\[
\Gamma\in\mathcal A(\mu,\nu)
\]
and
\[
\mathcal J_{\beta,\varepsilon,T}(\Gamma)
\leq
\SBB_{\beta,\varepsilon,T}(\mu,\nu)+\delta.
\]
Then by definition of infimum, $\SBB_{\beta,\varepsilon,T}(\mu,\nu) = \inf_{\Gamma\in\mathcal A(\mu,\nu)} \mathcal J_{\beta,\varepsilon,T}(\Gamma).$
\end{proof}

\subsubsection{Proof for Theorem \Cref{thm:projection-main}}
\begin{proof}
If the right-hand side is \(+\infty\), the claim is immediate.
Otherwise fix $\Gamma\in\mathcal A(\mu,\nu)$
with $\mathcal J_{\beta,\varepsilon,T}(\Gamma)<\infty$. Take $(X,Y,Z)\sim\Gamma$.

Let \(\Gamma^G\) denote the jointly Gaussian law on
\((\mathbb R^d)^3\) having exactly the same joint mean vector and joint
covariance matrix as \(\Gamma\). Take $(X^G,Y^G,Z^G)\sim\Gamma^G$. We will show that
\[
\Gamma^G\in\mathcal A(\mu^G,\nu)
\]
and
\[
\mathcal J_{\beta,\varepsilon,T}(\Gamma^G)
\leq
\mathcal J_{\beta,\varepsilon,T}(\Gamma).
\]

\paragraph{Part 1.}
Since \(X^G\) is Gaussian with the same mean and covariance as \(X\),
\[
X^G\sim \mu^G.
\]
Likewise \(Z^G\) is Gaussian and has the same mean and covariance as
\(Z\). Since $Z\sim\nu$ and \(\nu\) is already Gaussian,
\[
Z^G\sim\nu.
\]

Now we prove the preservation of the conditional-mean constraint. Set
\[
D:=Y-Z.
\]
Since
\[
\mathbb E[D\mid X]=0,
\]
we obtain
\[
\mathbb E[D]=0.
\]
Then
\begin{align*}
\Cov(D,X)
&=
\mathbb E
\left[
D(X-\mathbb E[X])^\top
\right]
\nonumber\\
&=
\mathbb E
\left[
\mathbb E[D\mid X]
(X-\mathbb E[X])^\top
\right]
\nonumber\\
&=
0.
\label{eq:D-X-cov-zero}
\end{align*}

Define
\[
D^G:=Y^G-Z^G.
\]
Since \(\Gamma^G\) has the same first and second moments as \(\Gamma\),
\[
\mathbb E[D^G]=0,
\qquad
\Cov(D^G,X^G)=0.
\]
The pair \((D^G,X^G)\) is jointly Gaussian. Zero cross-covariance
therefore implies independence. Hence
\[
\mathbb E[D^G\mid X^G]
=
\mathbb E[D^G]
=
0.
\]
Thus
\[
\mathbb E[Y^G\mid X^G]
=
\mathbb E[Z^G\mid X^G].
\]
Therefore
\[
\Gamma^G\in\mathcal A(\mu^G,\nu).
\]

\paragraph{Part 2.}
Now we work on two components of $\mathcal J_{\beta,\varepsilon,T}$. For the Wasserstein part, because \(\Gamma\) and \(\Gamma^G\) have the same joint first and second
moments,
\begin{align*}
\mathbb E_{\Gamma^G}\|Y^G-Z^G\|^2
&=
\left\|
\mathbb E_\Gamma[Y]-\mathbb E_\Gamma[Z]
\right\|^2 + \Tr\left(\Cov(Y)+\Cov(Z)-\Cov(Y,Z)-\Cov(Z,Y)\right)\\
&=\mathbb E_\Gamma\|Y-Z\|^2.
\label{eq:quadratic-term}
\end{align*}

For the relative entropy term, we want to prove that
$$
H\!\left(\Gamma^G_{X^GY^G}\,\middle|\, \mu^G(dx)\gamma_{x,\tau}(dy)\right) \leq H\!\left(\Gamma_{XY} \,\middle|\,\mu(dx)\gamma_{x,\tau}(dy)\right).
$$

By chain rule for relative entropy, as in Proposition \ref{prop:triple-static}'s proof, we have
\begin{align*}
H\!\left(
\Gamma_{XY}
\,\middle|\,
\mu(dx)\gamma_{x,\tau}(dy)
\right)
& =\int H(\kappa_x\mid\gamma_{x,\tau})\,\mu(dx)
\end{align*}
Similarly, since we have already shown that $\Gamma^G\in\mathcal A(\mu^G,\nu)$, the same disintegration argument yields
\begin{align*}
H\!\left(
\Gamma^G_{X^GY^G}
\,\middle|\,
\mu^G(dx)\gamma_{x,\tau}(dy)
\right)
& =\int H(\kappa^G_x\mid\gamma_{x,\tau})\,\mu^G(dx)
\end{align*}

Note that $\kappa^G_x$ and $\gamma_{x,\tau}$ are two Gaussians. It is the fact that for two Gaussians $\mathcal{N}(m_1, \Sigma_1)$ and $\mathcal{N}(m_2, \Sigma_2)$, the relative entropy has the explicit form
$$
H(\mathcal{N}(m_1, \Sigma_1) \mid \mathcal{N}(m_2, \Sigma_2)) = \frac{1}{2}\left[\operatorname{tr}\left(\Sigma_2^{\dagger} \Sigma_1\right)+\left(m_2-m_1\right)^T \Sigma_2^{\dagger}\left(m_2-m_1\right)-d+\ln \left(\frac{\operatorname{det}^* \Sigma_2}{\operatorname{det}^* \Sigma_1}\right)\right]
$$
where $\operatorname{det}^*$ is pseudodeterminant, i.e. the product of the nonzero eigenvalues. Then the relative entropy function $H\left(\kappa_{x}^G \mid \gamma_{x, \tau}\right)$ a quadratic polynomial in $x$. Because $\mu$ and $\mu^G$ share identical first and second moments, their integrals against any quadratic polynomial are exactly equal, i.e.
$$
\int H(\kappa^G_x\mid\gamma_{x,\tau})\,\mu^G(dx) = \int H(\kappa^G_x\mid\gamma_{x,\tau})\,\mu(dx)
$$

Lemma \ref{lem:gaussian-projection-entropy} below also gives
$$
\int H(\kappa_x\mid\gamma_{x,\tau})\,\mu(dx) \geq \int H(\kappa^G_x\mid\gamma_{x,\tau})\,\mu(dx)
$$
Combining the preceding identities with this inequality, we obtain
\begin{equation*}
\boxed{
H\!\left(
\Gamma^G_{X^GY^G}
\,\middle|\,
\mu^G(dx)\gamma_{x,\tau}(dy)
\right)
\leq
H\!\left(
\Gamma_{XY}
\,\middle|\,
\mu(dx)\gamma_{x,\tau}(dy)
\right).
}
\label{eq:entropy-final-compare}
\end{equation*}

Combining two parts,
\[
\mathcal J_{\beta,\varepsilon,T}(\Gamma^G)
\leq
\mathcal J_{\beta,\varepsilon,T}(\Gamma).
\]
Since
\[
\Gamma^G\in\mathcal A(\mu^G,\nu),
\]
it implies
\[
\SBB_{\beta,\varepsilon,T}(\mu^G,\nu)
\leq
\mathcal J_{\beta,\varepsilon,T}(\Gamma^G)
\leq
\mathcal J_{\beta,\varepsilon,T}(\Gamma).
\]
Taking the infimum over all finite-cost $\Gamma\in\mathcal A(\mu,\nu)$ yields
\[
\SBB_{\beta,\varepsilon,T}(\mu^G,\nu)
\leq
\SBB_{\beta,\varepsilon,T}(\mu,\nu).
\]
\end{proof}

\begin{lemma}
\label{lem:gaussian-projection-entropy}
Let $\mu$ be a probability measure on $\mathcal X$, and for
$\mu$-almost every $x\in\mathcal X$, let $p_x$ be a probability
measure on $\mathbb R^d$ with finite second moment. Let $p^G_{x}$ be the Gaussian measure having the same mean and covariance as
$p_x$. Suppose that $q_x $ is itself Gaussian. Then, whenever the relative entropies are well defined,
\[
H(p_x\mid q_x)
=
H(p_x\mid p^G_{x})
+
H(p^G_{x}\mid q_x).
\]
In particular,
\[
H(p_x\mid q_x)
\geq
H(p^G_{x}\mid q_x)
\qquad
\text{for $\mu$-a.e. }x.
\]
Consequently,
\[
\boxed{
\int_{\mathcal X} H(p_x\mid q_x)\,\mu(dx)
\geq
\int_{\mathcal X} H(p^G_{x}\mid q_x)\,\mu(dx).
}
\]
Equality holds if and only if
\[
H(p_x\mid p^G_{x})=0
\]
for $\mu$-almost every $x$, i.e., $p_x=p^G_{x}$ for
$\mu$-almost every $x$.
\end{lemma}

\begin{proof}
For $\mu$-almost every $x$, we have
\[
\begin{aligned}
H(p_x\mid q_x)-H(p_x\mid p^G_{x})
&=
\int_{\mathbb R^d}
\log\left(
\frac{dp^G_{x}}{dq_x}(y)
\right)p_x(dy).
\end{aligned}
\]
Since both $p^G_{x}$ and $q_x$ are Gaussian,
\[
y\longmapsto
\log\left(
\frac{dp^G_{x}}{dq_x}(y)
\right)
\]
is a polynomial of degree at most two in $y$. Since $p_x$ and
$p^G_{x}$ have the same first and second moments, it follows that
\[
\int_{\mathbb R^d}
\log\left(
\frac{dp^G_{x}}{dq_x}
\right)\,dp_x
=
\int_{\mathbb R^d}
\log\left(
\frac{dp^G_{x}}{dq_x}
\right)\,dp^G_{x}.
\]
The right-hand side is precisely
\[
H(p^G_{x}\mid q_x).
\]
Hence
\[
H(p_x\mid q_x)
=
H(p_x\mid p^G_{x})
+
H(p^G_{x}\mid q_x).
\]
The desired inequality follows from the non-negativity of relative
entropy $H(p_x\mid p^G_{x})\geq0$. Integrating the pointwise inequality with respect to $\mu(dx)$ gives
\[
\int_{\mathcal X}H(p_x\mid q_x)\,\mu(dx)
\geq
\int_{\mathcal X}H(p^G_{x}\mid q_x)\,\mu(dx).
\]
Finally, equality holds precisely when
$H(p_x\mid p^G_{x})=0$ for $\mu$-almost every $x$, which is equivalent
to $p_x=p^G_{x}$ $\mu$-almost everywhere.
\end{proof}

\subsubsection{Proof for Corollary \ref{cor:gaussian-main}}
\begin{proof}
Applying Theorem~\ref{thm:projection-main} to each pair $(\rho, \mu_i)$,
\[
\SBB_{\beta,\varepsilon,T}
(\rho^G,\mu_i)
\leq
\SBB_{\beta,\varepsilon,T}
(\rho,\mu_i) \qquad \forall i=1,\dots,n.
\]
Multiplying by \(\lambda_i\) and summing yields the first conclusion.

Since Gaussian laws form a subset of
\(\mathcal P_2(\mathbb R^d)\),
\[
\inf_{\mu\in\mathcal P_2(\mathbb{R^d})} \Bdec(\rho)
\leq
\inf_{m\in\mathbb R^d, S\in\mathbb S_+^d}
\Bdec(\mathcal N(m,S)).
\]
Conversely, for every $\rho$,
\[
\Bdec(\rho^G)
\leq
\Bdec(\rho).
\]
Taking the infimum over $\rho$ gives the reverse inequality, proving the second conclusion. Finally, if $\rho^\star$ minimizes $\Bdec$, then
\[
\inf \Bdec
\leq
\Bdec(\rho^{G\star})
\leq
\mathcal F(\rho^\star)
=
\inf\Bdec,
\]
so $\rho^{G\star}$ is also a minimizer.
\end{proof}

\subsection{Proof for \Cref{thm:covaction-main}}

\begin{proof}
    Let $X_t$ be any admissible semimartingale with finite cost. Define
\[
m_t=\E[X_t],
\qquad
\Sigma_t=\Cov(X_t).
\]
Since
\[
\dd X_t=\alpha_t\dd t+\sigma_t\dd W_t,
\]
we have
\[
\dot m_t=\E[\alpha_t].
\]
Therefore, by Jensen,
\[
\E\|\alpha_t\|^2\geq\|\dot m_t\|^2.
\]
Thus
\[
\int_0^T\frac12\E\|\alpha_t\|^2\dd t
\geq
\int_0^T\frac12\|\dot m_t\|^2\dd t
\geq
\frac{\|m_1-m_0\|^2}{2T},
\]
with equality when $\dot m_t$ is constant in time i.e. $\dot m_t=\frac{m_\nu-m_\mu}{T}$ or
\[
m_t=(1-t/T)m_\mu+(t/T)m_\nu.
\]
Now center the process:
\[
\tilde X_t=X_t-m_t,
\qquad
\tilde\alpha_t=\alpha_t-\dot m_t.
\]
Let
\[
A_t=\E[\sigma_t\sigma_t^\top],
\qquad
C_t=\E[\tilde\alpha_t\tilde X_t^\top].
\]
Differentiating the covariance gives
\[
\dot\Sigma_t=C_t+C_t^\top+A_t.
\]
So, with
\[
S_t=\dot\Sigma_t-A_t,
\]
we need
\[
C_t+C_t^\top=S_t.
\]
For any random vectors $U,V$ with covariance $\Sigma=\E[VV^\top]$, the regression inequality gives
\[
\E\|U\|^2\geq\Tr(C\Sigma^{-1}C^\top),
\qquad
C=\E[UV^\top].
\]
Applying this to $U=\tilde\alpha_t$, $V=\tilde X_t$,
\[
\E\|\tilde\alpha_t\|^2\geq\Tr(C_t\Sigma_t^{-1}C_t^\top).
\]
Minimizing over all $C_t$ such that $C_t+C_t^\top=S_t$ gives
\[
\E\|\tilde\alpha_t\|^2\geq\Gamma_{\Sigma_t}(S_t).
\]
Indeed, write $C=K\Sigma$. Then
\[
C+C^\top=K\Sigma+\Sigma K^\top.
\]
The minimizer is the symmetric $K$ solving
\[
K\Sigma+\Sigma K=S,
\]
and the minimum is
\[
\Tr(K\Sigma K)=\Gamma_\Sigma(S).
\]
Next, consider the diffusion term,
\[
\E\|\sigma_t-\sqrt{\varepsilon}\Id\|_F^2
=\Tr(A_t)-2\sqrt{\varepsilon}\,\E\Tr(\sigma_t)+\varepsilon d.
\]
For fixed $A_t=\E[\sigma_t\sigma_t^\top]$, the quantity is minimized by deterministic symmetric volatility
\[
\sigma_t=A_t^{1/2},
\]
giving
\[
\E\|\sigma_t-\sqrt{\varepsilon}\Id\|_F^2
\geq
\|A_t^{1/2}-\sqrt{\varepsilon}\Id\|_F^2.
\]
Combining the drift and diffusion lower bounds yields
\[
\mathsf{SBB}(\mu,\nu)
\geq
\frac{\|m_\nu-m_\mu\|^2}{2T}+\G_{\beta,\varepsilon,T}(\Sigma_\mu,\Sigma_\nu).
\]
It remains to show equality for Gaussian endpoints. Choose optimal paths $(\Sigma_t,A_t)$ for the covariance problem and define $K_t$ by
\[
K_t\Sigma_t+\Sigma_tK_t=\dot\Sigma_t-A_t.
\]
Consider the affine Gaussian SDE
\[
\dd X_t=\dot m_t\dd t+K_t(X_t-m_t)\dd t+A_t^{1/2}\dd W_t,
\qquad
X_0\sim\N(m_\mu,\Sigma_\mu).
\]
Then $X_t$ remains Gaussian, its mean is $m_t$, and its covariance satisfies
\[
\dot\Sigma_t=K_t\Sigma_t+\Sigma_tK_t+A_t.
\]
Thus $X_T\sim\N(m_\nu,\Sigma_\nu)$. Its cost is exactly the lower bound. Therefore equality holds.
\end{proof}

\clearpage

\clearpage
\section{DUAL POTENTIALS AND THE PARTICLE SCHEME}
\label{app:particles}
\subsection{Dual Formulation of SBB}

Thanks to \cite{alouadi2026lightsbb}, the pairwise SBB dual has the form
\begin{equation}\label{eq:sbb-dual}
\SBB(\mu_0,\mu_T)
=
\sup_{(\psi,v)}
\left\{
\int \psi(x)\,\mu_T(\dd x)
-
\int v(0,x)\,\mu_0(\dd x)
\right\},
\end{equation}
where $v$ solves a Hamilton--Jacobi--Bellman equation
\begin{equation}\label{eq:HJB}
\partial_t v + H_{\beta,\eps}^*(\nabla v,D^2v)=0,
\qquad
v(T,\cdot)=\psi,
\end{equation}
under the curvature constraint
\begin{equation*}
D^2v(t,x) \prec \beta I.
\end{equation*}
The optimal feedback controls are
\begin{equation*}\label{eq:controls-v}
\alpha_t^*(x)=\nabla v^*(t,x),
\qquad
\sigma_t^*(x)=\sqrt{\eps}\left(I-\frac1\beta D^2v^*(t,x)\right)^{-1}.
\end{equation*}

The reduced SBB system introduces heat potentials $h_t^*$ and maps $Y_t$ through
\begin{equation*}\label{eq:Phi-def}
\Phi_t(y)=\frac{|y|^2}{2}+\frac{\eps}{\beta}\log h_t^*(y),
\qquad
Y_t=(\nabla\Phi_t)^{-1}.
\end{equation*}
Equivalently,
\begin{equation*}\label{eq:Y-v-identity}
Y_t(x)=x-\frac1\beta \nabla v^*(t,x).
\end{equation*}
Hence
\begin{equation}\label{eq:grad-v-Y}
\nabla v^*(t,x)=\beta\bigl(x-Y_t(x)\bigr).
\end{equation}
The optimal controls can also be written as
\begin{equation}\label{eq:controls-h}
\alpha_t^*(x)=\eps\nabla\log h_t^*\bigl(Y_t(x)\bigr),
\qquad
\sigma_t^*(x)=\sqrt{\eps}\,D^2\Phi_t\bigl(Y_t(x)\bigr).
\end{equation}
Combining \eqref{eq:grad-v-Y} and \eqref{eq:controls-h} gives the computationally crucial identity
\begin{equation*}\label{eq:key-force-identity}
\boxed{
\nabla v^*(t,x)
=\beta\bigl(x-Y_t(x)\bigr)
=\eps\nabla\log h_t^*\bigl(Y_t(x)\bigr).
}
\end{equation*}

\subsection{First Variation of SBB}
Our result holds under the regularity assumptions made on the dual solutions. We leave it to future work to precisely quantify under what conditions on the measures those assumptions will hold.

We first show that optimal dual potentials provide convex subgradients with respect to linear perturbations of the marginals. Recall that for $F$ to be a functional on probability measures, A function $f_\rho$
is a \emph{linear functional subgradient (or convex subgradient)} at $\rho$ if
\begin{equation*}
  F(\nu)-F(\rho)\ge\int f_\rho\,d(\nu-\rho)
  \label{eq:sbb-linear-subgradient}
\end{equation*}
for every measure $\nu$ in the domain under consideration.
We write $f_\rho\in\partial_{\mathrm{lin}}F(\rho)$. 

\subsubsection*{Encoder Marginal}
Fix $\mu$ and define $F_{\mu}(\rho):=\SBB(\mu,\rho)$. The dual representation can be written as
\begin{equation}\label{eq:terminal-conjugate}
F_{\mu}(\rho)
=
\sup_{\psi}\left\{
\langle \psi,\rho\rangle-\mathcal A_{\mu}(\psi)
\right\},
\end{equation}
where $\mathcal A_{\mu}$ is the value term induced by the HJB solution at time $0$. Indeed, fix $\mu$ in the pairwise SBB dual \Cref{eq:sbb-dual}. For each terminal potential $\psi$, solve the HJB backward and call the solution $v^\psi$. Define $\mathcal{A}_{\mu}(\psi):=\int v^\psi(0, x) \mu(d x)$; then the dual becomes \Cref{eq:terminal-conjugate}

Let $\psi^*$ be an active maximizer at $\rho$. Then for any perturbation $\tilde\rho$,
\begin{align*}
F_{\mu}((1-\theta)\rho+\theta\tilde\rho)
&\ge
\langle \psi^*,(1-\theta)\rho+\theta\tilde\rho\rangle
-\mathcal A_{\mu}(\psi^*)\\
&=
F_{\mu}(\rho)
+\theta\int \psi^*\,\dd(\tilde\rho-\rho).
\end{align*}
Thus $\psi^*$ is a convex subgradient of $F_{\mu}$ at $\rho$.  

\subsubsection*{Decoder Marginal}

Fix $\mu$ and define $G_\mu(\rho):=\SBB(\rho, \mu)$. From the two-marginal dual,
\begin{equation*}\label{eq:source-dual}
G_\mu(\rho) = \sup_{(\psi,v)} \left\{
\langle \psi,\mu\rangle-\langle v(0,\cdot),\rho\rangle
\right\}.
\end{equation*}
Let $(\psi^*,v^*)$ be an active dual pair at $\rho$.  Then
\begin{align*}
G_\mu((1-\theta)\rho+\theta\tilde\rho)
&\ge
\langle \psi^*,\mu\rangle
-
\langle v^*(0,\cdot),(1-\theta)\rho+\theta\tilde\rho\rangle\\
&=G_\mu(\rho)
-\theta\int v^*(0,x)\,\dd(\tilde\rho-\rho)(x).
\end{align*}
Thus $-v^*(0,\cdot)$ is a subgradient.  

We now state assumptions under which a convex subgradient could be used to obtain the \emph{first variation}. Recall that A \emph{first variation} is a function $f_\rho$ satisfying
\begin{equation}
  \lim_{\theta\downarrow0}
  \frac{F((1-\theta)\rho+\theta\nu)-F(\rho)}{\theta}
  =
  \int f_\rho\,d(\nu-\rho)
  \label{eq:sbb-first-variation-definition}
\end{equation}
for every admissible direction $\nu-\rho$.  

Let $\mathcal U\subset\mathcal P_2$ be a $W_2$-open neighborhood of $\rho$, and assume:
\begin{enumerate}
\item \textbf{Finite values and attained duals.}
  The dual solution exists for every $\sigma\in\mathcal U$, $F(\sigma)<\infty$. Select its active marginal coefficient $f_\sigma$ and normalize
  it by $f_\sigma(0)=0$.
  \item \textbf{Uniform spatial regularity.}
  The selected coefficients belong to $C^1(\mathbb R^d)$, and there
  are constants $L_x,C_0<\infty$ such that, for every
  $\sigma\in\mathcal U$ and $x,y\in\mathbb R^d$,
  \begin{equation}
    |\nabla f_\sigma(x)-\nabla f_\sigma(y)|
       \le L_x|x-y|,
    \qquad
    |\nabla f_\sigma(0)|\le C_0.
    \label{eq:sbb-spatial-regularity}
  \end{equation}
  \item \textbf{Stability with respect to the marginal.}
  If $\sigma_n,\rho\in\mathcal U$ and
  $W_2(\sigma_n,\rho)\to0$, then
  \begin{equation}
    \nabla f_{\sigma_n}\longrightarrow\nabla f_\rho
    \quad\text{locally uniformly on }\mathbb R^d.
    \label{eq:sbb-marginal-stability}
  \end{equation}
\end{enumerate}

The normalization and \eqref{eq:sbb-spatial-regularity} imply
\begin{equation}
  |\nabla f_\sigma(x)|\le C_0+L_x|x|,
  \qquad
  |f_\sigma(x)|\le C_0|x|+\frac{L_x}{2}|x|^2.
  \label{eq:sbb-growth}
\end{equation}
Consequently, all selected coefficients are integrable against
measures in $\mathcal P_2$, and
$\nabla f_\rho\in L^2(\rho;\mathbb R^d)$.
Furthermore, \eqref{eq:sbb-marginal-stability} and the normalization give $f_{\sigma_n}\to f_\rho$ locally uniformly.

When the dual is attained at both $\rho$ and $\sigma$, then we have two-sided dual bound $\langle f_\rho,\sigma-\rho\rangle \le F(\sigma)-F(\rho) \le\langle f_\sigma,\sigma-\rho\rangle.$  Now define
$\rho_\theta=(1-\theta)\rho+\theta\nu$.
Since
$W_2^2(\rho,\rho_\theta)\le\theta W_2^2(\rho,\nu)$,
the measures $\rho_\theta$ lie in $\mathcal U$ for sufficiently
small $\theta>0$. The two-sided bound gives
\[
  \int f_\rho\,d(\nu-\rho)
  \le
  \frac{F(\rho_\theta)-F(\rho)}{\theta}
  \le
  \int f_{\rho_\theta}\,d(\nu-\rho).
\]
The normalized coefficients $f_{\rho_\theta}$ converge locally uniformly to $f_\rho$. By \eqref{eq:sbb-growth}, they are dominated in absolute value by a common quadratic function, integrable against $\rho+\nu$. Thus the integral on the right converges to the integral on the left. This gives the first variation.

Put everything together, we get
\begin{equation*}\label{eq:first-var-terminal}
\boxed{
\frac{\delta}{\delta\rho}\SBB(\mu,\rho)=\psi_{\mu\to\rho}^*
\qquad\text{up to an additive constant.}
}
\end{equation*}
and therefore
\begin{equation*}\label{eq:wgrad-terminal}
\nabla_W \SBB(\mu,\rho)(x)=\nabla\psi_{\mu\to\rho}^*(x).
\end{equation*}
Also,
\begin{equation*}\label{eq:first-var-source}
\boxed{
\frac{\delta}{\delta\rho}\SBB(\rho,\mu)
=-v^*_{\rho\to\mu}(0,\cdot)
\qquad\text{up to an additive constant.}
}
\end{equation*}
and therefore
\begin{equation*}\label{eq:wgrad-source}
\nabla_W\SBB(\rho,\mu)(x)=-\nabla v^*_{\rho\to\mu}(0,x).
\end{equation*}

\clearpage
\subsection{Pseudocode} \label{app:pseudocode}
\begin{breakablealgorithm}
\caption{SBB-Barycenters Hybrid Wasserstein Gradient Particle Descent Scheme}
\label{alg:sbb-symmetric-barycenter-particles}
\begin{algorithmic}[1]
\Require Data measures $\{\mu_i\}_{i=1}^N$ with weights $\{\lambda_i\}_{i=1}^N$; $\beta,\varepsilon,T$ such that $\beta T > 1$; Number of barycenter particles $M$; Outer iterations $K$; Step sizes $\{\eta_k\}_{k=0}^{K-1}$; (Optional) Terminal smoothing schedule $\{\delta_k\}_{k=0}^{K-1}$, with $0 < \delta_k \ll T$; Tolerance $\epsilon$; \textsc{Pairwise\_SBB\_Solver} (e.g. $\textsc{LightSBB}(\mu_{\mathrm{src}},\mu_{\mathrm{tgt}};\texttt{warm\_start})$) which returns a representation of the pairwise bridge; \textsc{Mode} $\in$ \{Decoder, Encoder, Symmetric, Markovian\}

\State Initialize particles $X_1^0,\dots,X_M^0$ for $\rho^0$
\State (Optional) Initialize warm starts $W_{i,\mathrm{dec}}$ and $W_{i,\mathrm{enc}}$ for all $i$

\For{$k=0,1,\dots,K-1$}

    \State Form the empirical barycenter iterate
    \[
        \rho^k = \frac{1}{M}\sum_{m=1}^M \delta_{X_m^k}.
    \]

    \ForAll{$i=1,\dots,N$ \textbf{in parallel}}
        \If{\textsc{Mode} $\in$ \{Decoder, Symmetric, Markovian\}}
        \State \Comment{Decoder bridge: $\rho^k \to \mu_i$}
        \State $\texttt{BridgeDec}_i \gets 
         \textsc{Pairwise\_SBB\_Solver}(\rho^k,\mu_i)$
        \ElsIf{\textsc{Mode} $\in$ \{Encoder, Symmetric, Markovian\}}
        \State \Comment{Encoder bridge: $\mu_i \to \rho^k$}
        \State $\texttt{BridgeEnc}_i \gets
         \textsc{Pairwise\_SBB\_Solver}(\mu_i,\rho^k)$
         \EndIf
    \EndFor

    \For{$m=1,\dots,M$}
        \State $\mathfrak v _{\mathrm{dec}} \gets 0, \mathfrak v_{\mathrm{enc}} \gets 0$
        \For{$i=1,\dots,N$}
            \If{\textsc{Mode} $\in$ \{Encoder, Symmetric, Markovian\}}
            \State $g_{\mathrm{enc},i} \gets \beta\bigl(X_m^k-Y_T^{(i,k), enc}(X_m^k)\bigr).$
            \State $\mathfrak v_{\mathrm{enc}} \gets \mathfrak v_{\mathrm{enc}} - \lambda_i g_{\mathrm{enc},i}$
            \ElsIf{\textsc{Mode} $\in$ \{Decoder, Symmetric, Markovian\}}
            \State $g_{\mathrm{dec},i} \gets \beta\bigl(X_m^k-Y_{0}^{(i,k), dec}(X_m^k)\bigr)$
            \State $\mathfrak v_{\mathrm{dec}} \gets \mathfrak v_{\mathrm{dec}} +\lambda_i g_{\mathrm{dec},i}$
            \EndIf
        \EndFor
        \If{\textsc{Mode} = Symmetric}
        \State $\mathfrak v_{\mathrm{sym}} \gets \frac{1}{2}\left(\mathfrak v_{\mathrm{dec}}+\mathfrak v_{\mathrm{enc}}\right)$
        \ElsIf{\textsc{Mode} = Markovian}
        \State $\mathfrak v_{\mathrm{mar}} \gets \lambda_2 g_{\mathrm{dec},2} - \lambda_1 g_{\mathrm{enc},1}$
        \EndIf
    \EndFor

    \For{$m=1,\dots,M$}
        \State \Comment{Particle update according to \textsc{Mode}}
        \State $X_m^{k+1} \gets X_m^k+\eta_k \mathfrak v_{\{\mathrm{enc}, \mathrm{dec}, \mathrm{sym}, \mathrm{mar}\}}$
    \EndFor

    \State \Comment{Stopping check}
    \If{$\frac{1}{M}\sum_{m=1}^M \lVert \mathfrak v_{\{\mathrm{enc}, \mathrm{dec}, \mathrm{sym}, \mathrm{mar}\}}\rVert^2 < \epsilon$}
        \State \textbf{break}
    \EndIf
\EndFor

\Ensure Final particle measure
\[
    \rho^K = \frac{1}{M}\sum_{m=1}^M \delta_{X_m^K}.
\]
\end{algorithmic}
\end{breakablealgorithm}

\begin{algorithm}[htp!]
\caption{Gaussian Specialization of Decoder SBB Barycenter}
\label{alg:gaussian-barycenter}
\begin{algorithmic}[1]
\Require Gaussian targets
$\{\nu_i=\mathcal N(m_i,\Sigma_i)\}_{i=1}^N$;
weights $\{\lambda_i\}_{i=1}^N$;
\textsc{Pairwise\_SBB\_Solver} (e.g. $\textsc{LightSBB}(\mu_{\mathrm{src}},\mu_{\mathrm{tgt}},\texttt{warm\_start})$);
initial covariance $S^0\succ0$;
step sizes $\{\eta_k\}_{k=0}^{K-1}$
\State Compute the barycenter mean
\[
\bar m\gets\sum_{i=1}^N\lambda_i m_i.
\]
\For{$k=0,\ldots,K-1$}
    \For{$i=1,\ldots,N$ \textbf{in parallel}}
        \State Solve the Gaussian SBB bridge
        \[
        \mathcal N(\bar m,S^k)
        \longrightarrow
        \mathcal N(m_i,\Sigma_i).
        \]
        \State Extract its initial linear drift coefficient $K_{i,0}^k$.
    \EndFor
    \State Aggregate
    \[
    \bar K_0^k
    \gets
    \sum_{i=1}^N\lambda_iK_{i,0}^k.
    \]
    \State Set
    \[
    L_k\gets I_d+\eta_k\bar K_0^k.
    \]
    \State Update the covariance
    \[
    S^{k+1}\gets L_kS^kL_k^\top.
    \]
\EndFor
\State \Return
\[
\rho^K=\mathcal N(\bar m,S^K).
\]
\end{algorithmic}
\end{algorithm}

\clearpage

\section{MEAN--COVARIANCE IMPLEMENTATION OF AFFINE PARTICLE UPDATES}
\label{app:moments}
This section gives the additional computational specialization of the particle scheme. The premise is that the pairwise oracle supplies affine update, as in the Gaussian dynamics of Section \ref{sec:gaussian}. 

\subsection{The general affine update}
Let the current measure have mean $m$ and covariance $S$, and suppose its chosen descent velocity is
\begin{equation*}\label{eq:affine-velocity-app}
 u(x)=b+C(x-m).
\end{equation*}
The particle update $X^+=X+\eta u(X)$ can be written as
\[
 X^+=m+\eta b+(I_d+\eta C)(X-m).
\]
Taking its expectation gives
\begin{equation}\label{eq:affine-mean-app}
 m^+=m+\eta b.
\end{equation}
Subtracting this mean from $X^+$ gives $X^+-m^+=(I_d+\eta C)(X-m)$. Hence
\begin{equation}\label{eq:affine-cov-app}
 S^+=(I_d+\eta C)S(I_d+\eta C)^\top.
\end{equation}
These are algebraic identities for any input law with a finite second moment. When the input is Gaussian, the output is also Gaussian, since the update is affine. For a finite cloud, they are exact for the empirical covariance computed with normalization $1/M$.

\subsection{Decoder SBB Gaussian Barycenter update}
For the bridge $\N(m,S)\to\N(m_i,\Sigma_i)$, write the Gaussian affine drift at the initial time as
\[
 g_{i,\mathrm{dec}}(x)=\frac{m_i-m}{T}+K_{i,0}(x-m).
\]
Using the decoder sign in the particle scheme gives
\[
 b_{\mathrm{dec}}=\frac{\bar m-m}{T},\qquad
 C_{\mathrm{dec}}=\sum_i\lambda_iK_{i,0},\qquad
 \bar m=\sum_i\lambda_i m_i.
\]
Equations~\eqref{eq:affine-mean-app}--\eqref{eq:affine-cov-app} give the two updates in the main paper. The covariance coefficient is determined by the fitted bridge, not by a Euclidean covariance gradient introduced in place of that bridge.

\section{EXPERIMENTAL DETAILS AND RESULTS}
\label{app:experiments}

\subsection{Complete toy visualizations}
We illustrate the directional geometry of SBB barycenters using three anisotropic Gaussian distributions in two dimensions, with equal weights $\lambda_i=1 / 3$. The reference setting uses 512 samples per distribution and 512 barycenter particles. Decoder, encoder, and symmetric objectives share the same input samples and Gaussian initialization. Pairwise models use $T=1, \beta=50, \epsilon=1$, and 32 Gaussian potential components. Adam uses a learning rate of $10^{-3}$, mini-batches of 256. Barycenter optimization runs for at most 50 iterations with step size $\eta_k=0.15(0.95)^k$ and a mean squared velocity norm stopping threshold of $10^{-3}$.

Figures~\ref{fig:toy-dec-app}--\ref{fig:toy-sym-app} provide the complete particle evolution for the three SBB barycenter orientations considered in the main text. 
All experiments start from the same three two-dimensional input measures and use the same initialization of the barycenter particles. Each figure contains three panels. 
The \emph{left panel} shows the three input measures, which are kept fixed throughout the optimization. 
The \emph{middle panel} shows the evolution of the barycenter particles during Wasserstein subgradient descent: the trajectories indicate how particles move from their initialization toward the final barycenter configuration. The \emph{right panel} overlays the final barycenter particles in black with the original input measures, making it possible to compare directly the location, shape, and spread selected by each orientation.
\begin{figure}[htbp]
\centering\includegraphics[width=\textwidth]{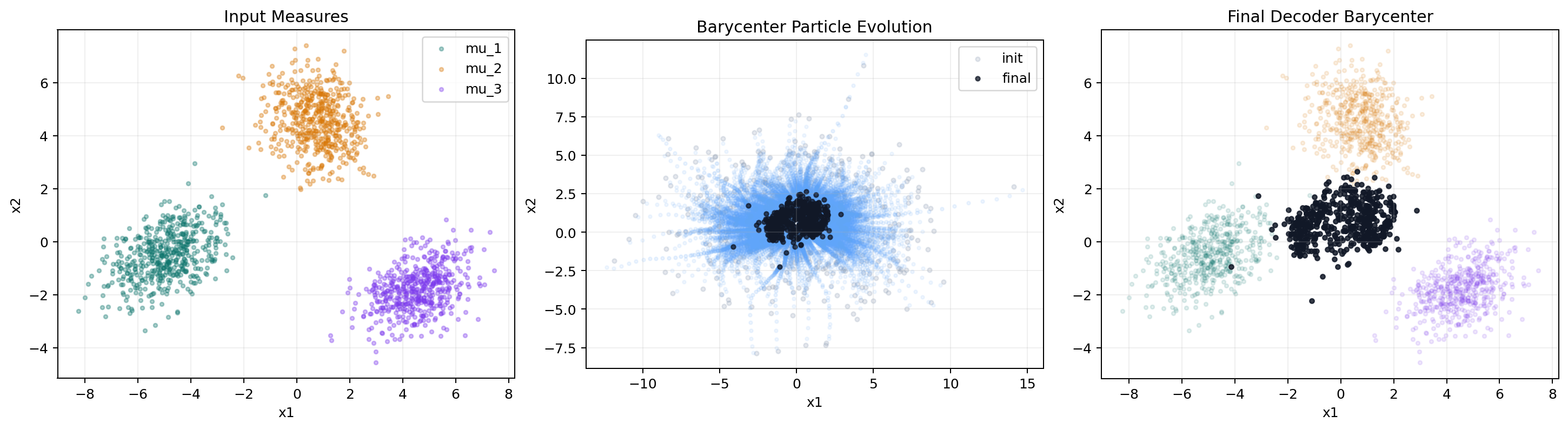}
\caption{Decoder Barycenter visualization}\label{fig:toy-dec-app}
\end{figure}
\begin{figure}[htbp]
\centering\includegraphics[width=\textwidth]{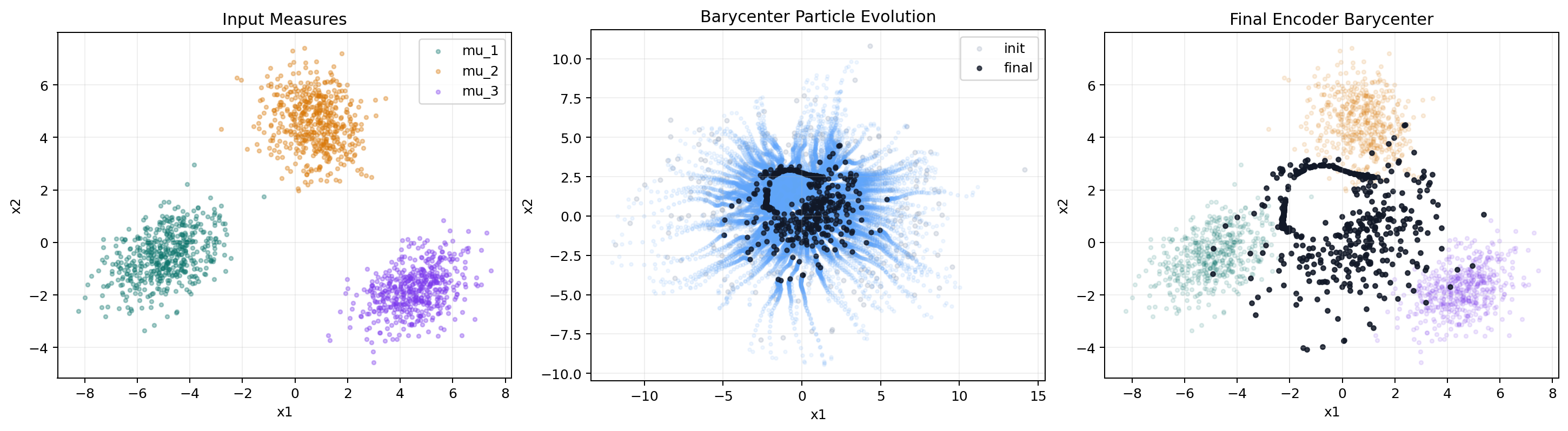}
\caption{Encoder Barycenter visualization}\label{fig:toy-enc-app}
\end{figure}
\begin{figure}[htbp]
\centering\includegraphics[width=\textwidth]{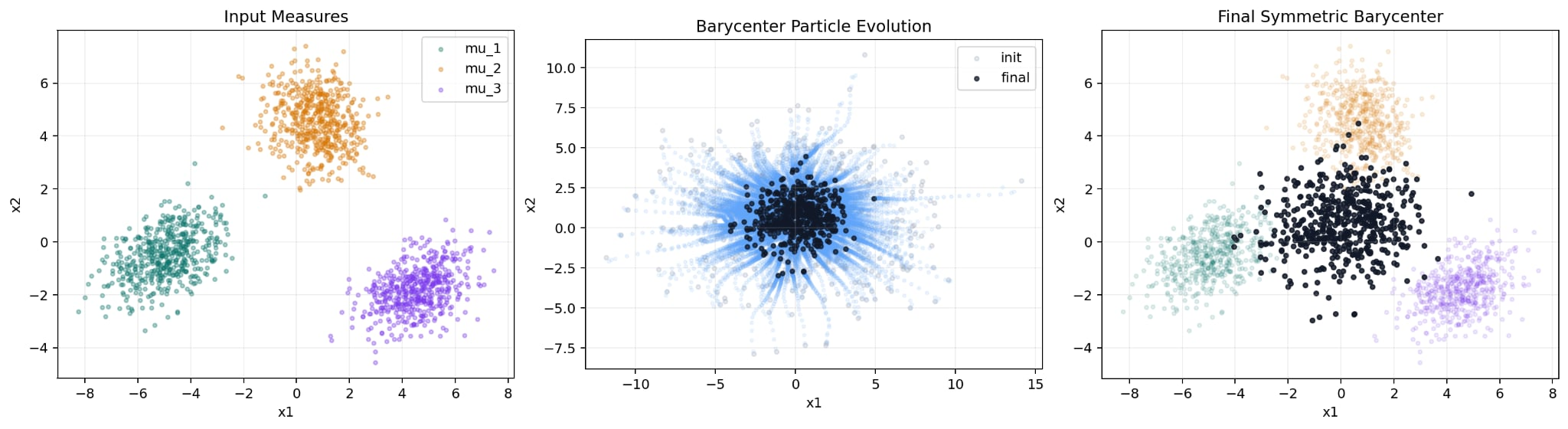}
\caption{Symmetric Barycenter visualization}\label{fig:toy-sym-app}
\end{figure}

\subsection{Spatial MNIST: protocol}
\paragraph{Probability-measure representation.}
We represent each digit image as a probability measure over the two-dimensional pixel plane. 
Given a pixel grid $G$, the nonnegative image intensities are normalized to define
\[
    \rho=\sum_{x\in G} p(x)\delta_x,
    \qquad
    \sum_{x\in G}p(x)=1.
\]
Under this representation, image mass is distributed over spatial locations rather than treated as a vector of pixel intensities. Consequently, a diffusion-like corruption is naturally modeled by spatial convolution of the measure with a Gaussian kernel, rather than by adding independent noise directly to the pixel values. Our experiment follows this viewpoint as a modification of the spatial barycenter experiment of \citet{cuturi2014fast}.

\paragraph{Latent sources and observations.}
We construct the experiment around a latent-source recovery task. We select one approximately centered clean image from each digit class as the latent source $\rho^\star$. From each latent source, we generate $N\in\{32,64,128\}$ corrupted observations, using $N=32$ as the default configuration. The original train/test split is kept fixed throughout the experiment, and the latent test images are never used for hyperparameter selection. Each observed measure is first obtained by applying a random geometric transformation to the latent source. On a $32\times32$ or $40\times40$ canvas, we define
\[
    T_i(x)
    =
    s_i R_{\vartheta_i}(x-x_0)+x_0+b_i,
\]
where $x_0$ denotes the image center, $s_i$ controls scaling, $R_{\vartheta_i}$ is a planar rotation, and $b_i$ is a translation. We sample translations from $[-6,6]^2$ and recenter them so that the average translation is approximately zero; one convenient construction is to use pairs $(b,-b)$. The scale is sampled as
\[
    s_i\sim\operatorname{Unif}(0.7,1.4),
\]
while the optional rotation ablation uses
\[
    \vartheta_i\sim\operatorname{Unif}(-20^\circ,20^\circ).
\]
A sufficiently padded canvas is used to limit boundary artifacts, and any mass loss or renormalization caused by the transformation is recorded. To ensure a fair comparison, all barycenter methods are evaluated using exactly the same geometric transformations and image-rendering procedure.

\paragraph{Blur regimes.}
After the deterministic transformation, we introduce spatial diffusion by convolving the transformed measure with a Gaussian kernel. We consider several regimes in order to separate the effects of deterministic transport and stochastic dispersion. In the common isotropic setting, all observations generated from the same latent source use the same covariance,
\[
    \Omega_i=\sigma^2 I_2,
\]
so that the stochastic component is shared across domains.

Our main SBB-specific setting instead uses heterogeneous anisotropic diffusion. For each observation, we define
\[
    \Omega_i
    =
    R_{\varphi_i}
    \diag\!\left(\sigma_{i,1}^2,\sigma_{i,2}^2\right)
    R_{\varphi_i}^{\top},
\]
where the principal blur scales and orientation vary across observations. In the default configuration, we use
\[
    \sigma_{i,1}\in[0.4,2.0],\qquad
    \sigma_{i,2}\in[0.2,1.2],
    \qquad
    \varphi_i\in[0,\pi].
\]
The resulting observed distribution is therefore
\[
    \nu_i
    =
    (T_i)_\#\rho^\star * \N(0,\Omega_i).
\]

\begin{figure*}[!t]
    \centering
    \includegraphics[width=1\linewidth]{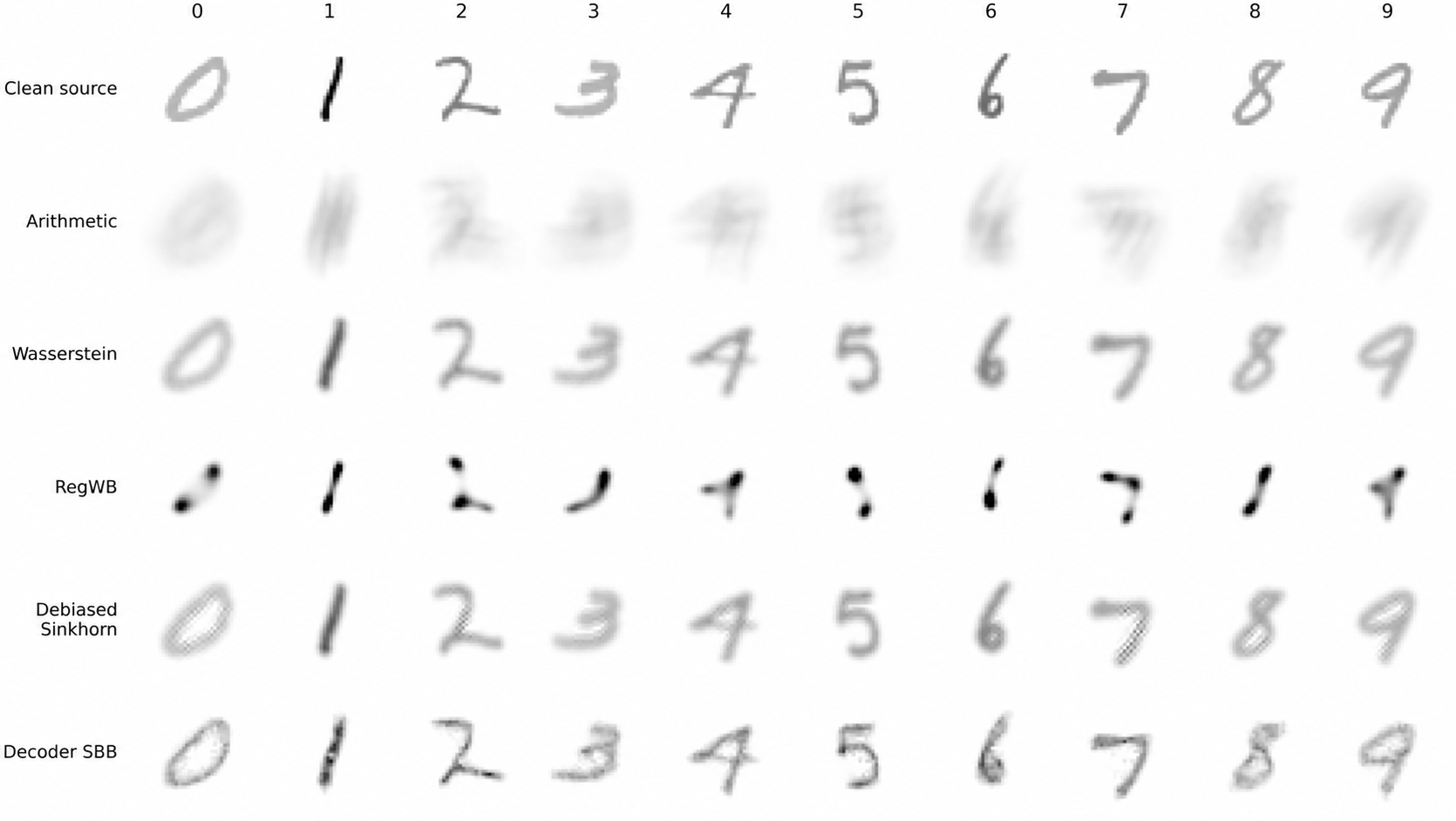}
    \caption{Qualitative latent-source recovery on MNIST.
Each column corresponds to one digit class, from $0$ to $9$.
The first row shows the clean latent source, while the remaining rows show the barycenters reconstructed by the arithmetic mean, Wasserstein barycenter, Regularized Wasserstein Barycenter, debiased Sinkhorn barycenter, and decoder SBB barycenter, respectively.
The decoder SBB reconstruction is generally sharper and more closely preserves the geometry of the clean source, consistent with the quantitative results in \Cref{tab:mnist-main}}
    \label{fig:mnist_visualization}
\end{figure*}

\paragraph{Hyperparameter Setting.}
The selected SBB parameters are $T=1, \beta=10$, and $\epsilon=$ 0.03 , equivalently $\theta=\beta T=10$ and $\tau=\epsilon T=0.03$. Each pairwise model uses 64 diagonal Gaussian potential components, diagonal initialization 0.1, and a learned inverse network with spatial width 64 and time-embedding width 16 . We use 4,096 barycenter particles and sample 4,096 points from each observation. Optimization performs exactly 95 outer updates with step size $\eta_k=0.1(0.995)^k$, with velocity-based early stopping disabled. Every outer update fits the pairwise models through two inner stages, each containing 100 gradient updates for the potential network and 100 for the inverse network. Both networks use Adam with learning rate 0.003 , minibatch size 256, and per-network gradient clipping at norm 1. The terminal smoothing fraction is 0.01 , sampling chunks contain 2,048 points, and the independent pairwise models are batched on the GPU in float32. Hyperparameters were selected using clean-training-image $W_2$ and frozen before test evaluation. The completed comparison uses one test image per digit 0-9, $W_2$ is estimated using 512 samples from each rendered image. The Wasserstein baseline uses exact transport subproblems with at most 100 outer iterations, while debiased Sinkhorn uses fixed regularization 0.03 and a maximum of 200,000 grid iterations. RegWB solves the static empirical objective on the full $40\times40$ pixel grid, using regularization $0.03$ in source-normalized coordinates, at most 4,000 iterative proportional fitting sweeps.

\subsection{Single-cell RNA: protocol}

\paragraph{Preprocessing.}
We normalize the expression matrix and embed all cells from the considered condition into a common $30$-dimensional PCA space. Each time point or cell subset is then represented by the empirical measure $P_t=\frac{1}{n_t}\sum_{j=1}^{n_t}\delta_{x_j^{(t)}}$. The same embedding is used for all competing methods. We analyze the serum and 2i conditions separately and subsample the same number of cells from each distribution when necessary to control computational cost.

\paragraph{Hyperparameter Setting}
\begin{itemize}
    \item \textbf{Common-ancestor reconstruction}. Evaluation uses MET cells from days 7.5-8.5. The supplied log-normalized expression values are preserved, 2,000 genes are selected by log-expression variance. Each run uses 512 cells per input population, 512 barycenter particles, and 512 evaluation cells. Particles are initialized from a Gaussian matching the weighted input mixture's mean and covariance, with covariance stabilization $10^{-4} I$. The selected pairwise parameters are $T=1, \beta=30$, and $\epsilon=0.3$, with 64 diagonal Gaussian components initialized at 0.1 , inverse-network spatial width 64 , time-embedding width 16 , terminal smoothing fraction 0.01 , and sampling chunk size 2,048. Adam uses learning rate 0.005 , minibatch size 256 , and gradient clipping at norm 1 . Each pairwise fit has two stages with 200 and 163 gradient updates, respectively, for each network. Barycenter optimization permits 75 outer iterations with constant step size 0.3 and a stopping threshold of $10^{-3}$ on the mean squared particle velocity. Pairwise network parameters are warm-started between outer iterations, while Adam states are reset for each pairwise fit. Wasserstein and Waddington-OT baselines permit 100 outer iterations with tolerance $10^{-6}$; Sinkhorn uses regularization 0.1 , the dual L-BFGS inner solver, a 10,000 -iteration inner cap, and inner tolerance $10^{-8}$. RegWB uses regularization $0.1$ in source-normalized coordinates, 65536 importance-quadrature nodes, and at most 12,000 iterative proportional fitting sweeps. 
    \item \textbf{Held-out population reconstruction: shared setup.}
We reconstruct the intermediate population at $t_0$ from populations at $t_-=t_0-\Delta/2$ and $t_+=t_0+\Delta/2$, where $\Delta\in\{1,2,3\}$ denotes the total endpoint separation in days. Consequently, the temporal fraction is
$\alpha=(t_0-t_-)/(t_+-t_-)=1/2$.
SBB Decoder and SBB Markovian use separately selected parameters and
adjusted barycenter weights
$\lambda_2=c\alpha/(1-\alpha+c\alpha)=c/(1+c)$ and
$\lambda_1=1-\lambda_2$.
We preserve the provided log-normalized expression, select 2,000 highly
variable genes per condition. The quantitative tuning and assessment runs use 512 cells per endpoint,
512 reconstructed particles, and 512 target cells.
Observed cells are sampled without replacement within each population.
The two SBB variants share an initialization sampled from the chronological
arithmetic mixture $(1-\alpha)P_{t_-}+\alpha P_{t_+}$, independently of $c$.
Unless specified otherwise below, each pairwise model uses time horizon
$T=1$ for all $\Delta$, 64 diagonal Gaussian potential components with
diagonal scale parameters initialized to $0.1$ and uniform initial
component weights, inverse-network spatial and temporal widths 64 and 16, terminal smoothing fraction $0.01$, and sampling chunks of 2,048 points. The full finite-$\beta$ inverse is used throughout.
Adam uses learning rate $0.003$, minibatch size 256, and gradient clipping at norm 1 per network. Each outer update contains two inner stages with 128 gradient steps. Networks are warm-started between outer iterations, whereas Adam is
reinitialized at each outer iteration. We use a constant barycenter step size $\eta$, a selected fixed number $L$ of outer iterations, and no velocity-based early stopping. Arithmetic mixture, Wasserstein displacement interpolation, and Waddington-OT use the chronological weights $(1/2,1/2)$. Waddington-OT uses regularization $0.05$, marginal-relaxation penalties 1 and 50, three growth iterations, the supplied growth-rate metadata, and the actual endpoint separation $\Delta$. RegWB uses chronological weights, regularization $0.1$, and analytic elimination of the central variable; the endpoint coupling is solved with an iteration cap of 12,000.
TrajectoryNet uses three hidden layers of width 64 with LeakyReLU, Adam with learning rate $10^{-3}$ and weight decay $5\times10^{-5}$, and minibatches of 1,000 cells sampled with replacement per endpoint.

\noindent\textbf{One-day intervals.}
We evaluate 64 midpoint triplets, comprising 32 triplets per condition, with $t_-=t_0-0.5$ and $t_+=t_0+0.5$. The global SBB Decoder configuration uses $\beta=100$, $\epsilon=0.3$,
$\eta=0.03$, $c=1$, and $L=60$; the global SBB Markovian configuration uses $\beta=100$, $\epsilon=0.1$, $\eta=0.03$, $c=2$, and $L=60$.

\noindent\textbf{Two-day intervals.}
We evaluate 66 midpoint triplets, comprising 33 triplets per condition, with $t_-=t_0-1$ and $t_+=t_0+1$. The global Decoder configuration uses $\beta=30$, $\epsilon=0.3$,
$\eta=0.1$, $c=1$, and $L=40$; the global Markovian configuration
uses $\beta=100$, $\epsilon=0.1$, $\eta=0.03$, $c=2$, and $L=40$. 

\noindent\textbf{Three-day intervals.}
We evaluate 62 midpoint triplets, comprising 31 triplets per condition, with $t_-=t_0-1.5$ and $t_+=t_0+1.5$. The global Decoder configuration uses $\beta=30$, $\epsilon=0.1$, $\eta=0.1$, $c=2$, and $L=60$ for serum and 40 for 2i; the global Markovian configuration uses $\beta=10$, $\epsilon=0.1$, $\eta=0.3$, $c=1$, and $L=60$ for serum and 40 for 2i. 
\end{itemize}

\clearpage
\subsubsection{1-day interval}
\begin{table}[htp!]
\centering
\small
\caption{Intermediate-state reconstruction 1-day intervals result.}
\label{tab:serum_2i_1day_comparison_overlap}
\resizebox{\linewidth}{!}{%
\begin{tabular}{lcccc}
\toprule
Interpolation Method & $W_2\downarrow$ & MMD$\downarrow$ & Cell-type TV$\downarrow$ & Signature MAE$\downarrow$ \\
\midrule
\multicolumn{5}{l}{\textbf{serum condition}} \\
\midrule
Arithmetic mixture
& $11.9916 \pm 0.1855$
& $0.4044 \pm 0.0120$
& $0.0588 \pm 0.0289$
& $0.0364 \pm 0.0062$ \\
Wasserstein
& $10.9134 \pm 0.2175$
& $0.4189 \pm 0.0155$
& $0.0859 \pm 0.0284$
& $0.0359 \pm 0.0008$ \\
RegWB
& $11.0323 \pm 0.2059$ 
& $0.4154 \pm 0.0151$
& $0.0892 \pm 0.0171$
& $0.0370 \pm 0.0029$ \\
TrajectoryNet
& $14.6255 \pm 0.2413$ 
& $0.5539 \pm 0.0123$
& $0.0658 \pm 0.0119$
& \textcolor{red}{$0.0329 \pm 0.0043$} \\
Waddington-OT
& $10.6621 \pm 0.0752$
& $0.4034 \pm 0.0080$
& $0.2383 \pm 0.0491$
& $0.0385 \pm 0.0036$ \\
SBB Decoder (ours)
& $10.3515 \pm 0.1308$
& $0.3778 \pm 0.0118$
& \textcolor{red}{$0.0534 \pm 0.0336$}
& $0.0375 \pm 0.0097$ \\
SBB Markovian (ours)
& \textcolor{red}{$10.1848 \pm 0.1203$}
& \textcolor{red}{$0.3408 \pm 0.0106$}
& $0.0553 \pm 0.0329$
& $0.0349 \pm 0.0072$ \\
\midrule
\multicolumn{5}{l}{\textbf{2i condition}} \\
\midrule
Arithmetic mixture
& $11.9001 \pm 0.6598$
& $0.1731 \pm 0.0196$
& $0.1536 \pm 0.0530$
& $0.0882 \pm 0.0422$ \\
Wasserstein
& $11.3029 \pm 0.7513$
& $0.1797 \pm 0.0234$
& $0.1745 \pm 0.0368$
& $0.1146 \pm 0.0458$ \\
RegWB
& $11.5855 \pm 0.7610$
& $0.1773 \pm 0.0247$
& $0.1816 \pm 0.0508$
& $0.1128 \pm 0.0392$ \\
TrajectoryNet
& $12.3068 \pm 0.4723$
& $0.1857 \pm 0.0110$
& $0.1700 \pm 0.0175$
& $0.1037 \pm 0.0071$ \\
Waddington-OT
& $12.6290 \pm 1.2882$
& $0.2239 \pm 0.0400$
& $0.1901 \pm 0.0587$
& $0.1024 \pm 0.0323$ \\
SBB Decoder (ours)
& \textcolor{red}{$10.9584 \pm 0.7523$}
& \textcolor{red}{$0.1572 \pm 0.0236$}
& \textcolor{red}{$0.1530 \pm 0.0536$}
& $0.0791 \pm 0.0384$ \\
SBB Markovian (ours)
& $11.5364 \pm 0.5676$
& $0.1698 \pm 0.0191$
& $0.1536 \pm 0.0530$
& \textcolor{red}{$0.0684 \pm 0.0374$} \\
\bottomrule
\end{tabular}%
}
\end{table}

\begin{figure}[htp!]
    \centering
    \subfloat{
        \includegraphics[width=1\linewidth]{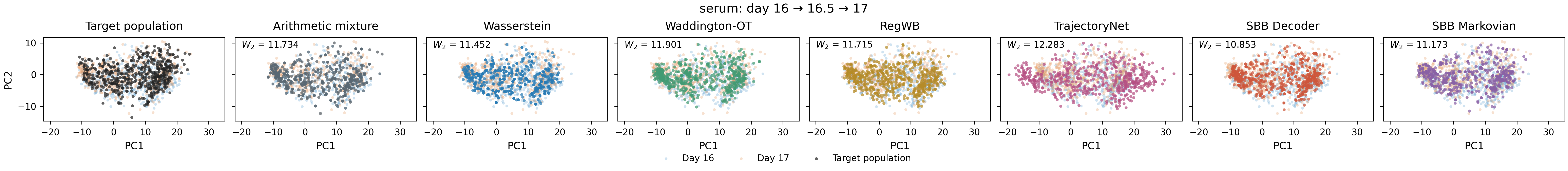}} \\
    \subfloat{
        \includegraphics[width=1\linewidth]{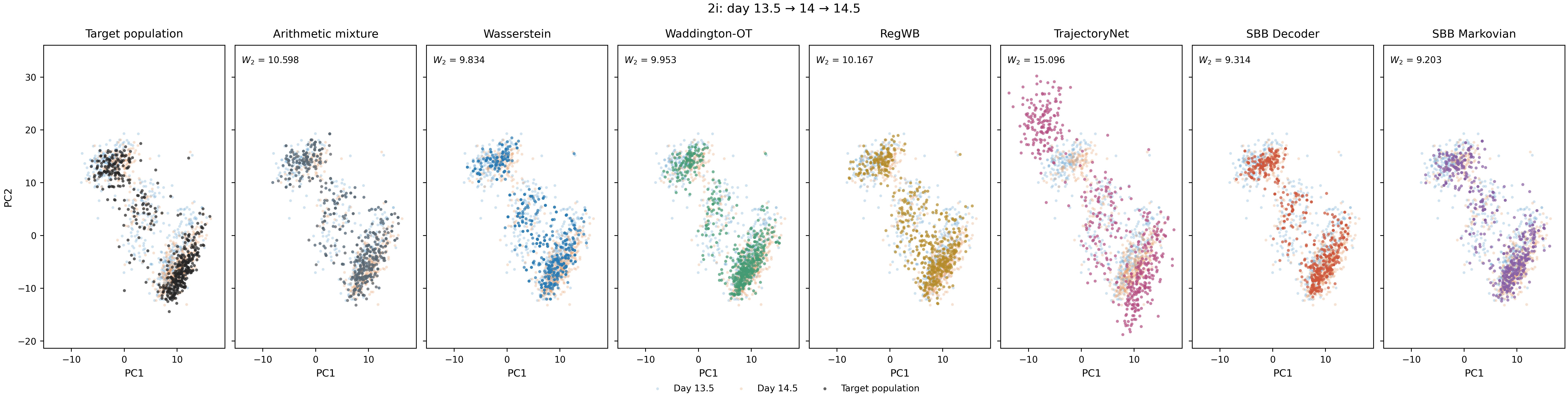}
    } 
    \caption{Some qualitative figures for reconstruction of intermediate cell populations using 1-day intervals.}
    \label{fig:1day_qualitative_reconstruction}
\end{figure}

\clearpage
\subsubsection{2-day interval}
\begin{table}[htp!]
\centering
\small
\caption{Intermediate-state 2-day intervals reconstruction result.}
\label{tab:serum_2i_2day_comparison}
\resizebox{\linewidth}{!}{%
\begin{tabular}{lcccc}
\toprule
Interpolation Method & $W_2\downarrow$ & MMD$\downarrow$ & Cell-type TV$\downarrow$ & Signature MAE$\downarrow$ \\
\midrule
\multicolumn{5}{l}{\textbf{serum condition}} \\
\midrule
Arithmetic mixture
& $12.5039 \pm 0.4085$
& $0.5773 \pm 0.0140$
& $0.2077 \pm 0.0502$
& $0.2725 \pm 0.0218$ \\
Wasserstein
& $12.2979 \pm 0.4909$
& $0.6103 \pm 0.0145$
& $0.1699 \pm 0.0052$
& $0.3029 \pm 0.0146$ \\
RegWB
& $12.4198 \pm 0.4824$
& $0.6058 \pm 0.0139$
& \textcolor{red}{$0.1654 \pm 0.0216$}
& $0.2993 \pm 0.0222$ \\
TrajectoryNet
& $14.1794 \pm 0.5810$
& $0.6762 \pm 0.0221$
& $0.3835 \pm 0.0329$
& $0.3046 \pm 0.0040$ \\
Waddington-OT
& $12.2477 \pm 0.3275$
& $0.6125 \pm 0.0102$
& $0.1660 \pm 0.0254$
& $0.3613 \pm 0.0135$ \\
SBB Decoder (ours)
& \textcolor{red}{$11.6635 \pm 0.3243$}
& $0.5858 \pm 0.0115$
& $0.2070 \pm 0.0559$
& $0.2672 \pm 0.0231$ \\
SBB Markovian (ours)
& $12.0708 \pm 0.3970$
& \textcolor{red}{$0.5652 \pm 0.0124$}
& $0.2077 \pm 0.0502$
& \textcolor{red}{$0.2491 \pm 0.0212$} \\
\midrule
\multicolumn{5}{l}{\textbf{2i condition}} \\
\midrule
Arithmetic mixture
& $13.3831 \pm 0.0608$
& $0.3970 \pm 0.0029$
& $0.2565 \pm 0.0147$
& $0.2382 \pm 0.0171$ \\
Wasserstein
& $13.3287 \pm 0.2158$
& $0.4305 \pm 0.0084$
& \textcolor{red}{$0.1224 \pm 0.0079$}
& $0.2386 \pm 0.0152$ \\
RegWB
& $13.4550 \pm 0.2619$
& $0.4251 \pm 0.0093$
& $0.1283 \pm 0.0236$
& $0.2412 \pm 0.0156$ \\
TrajectoryNet
& $14.9870 \pm 0.5248$
& $0.5024 \pm 0.0230$
& $0.2865 \pm 0.0049$
& $0.2030 \pm 0.0187$ \\
Waddington-OT
& $13.6357 \pm 0.2592$
& $0.4465 \pm 0.0079$
& $0.1465 \pm 0.0156$
& $0.2485 \pm 0.0279$ \\
SBB Decoder (ours)
& \textcolor{red}{$12.1139 \pm 0.1595$}
& $0.3859 \pm 0.0070$
& $0.2520 \pm 0.0085$
& \textcolor{red}{$0.1991 \pm 0.0294$} \\
SBB Markovian (ours)
& $12.6380 \pm 0.1817$
& \textcolor{red}{$0.3781 \pm 0.0074$}
& $0.2559 \pm 0.0160$
& $0.2219 \pm 0.0254$ \\
\bottomrule
\end{tabular}%
}
\end{table}

\begin{figure}[htp!]
    \centering
    \subfloat{
    \includegraphics[width=1\linewidth]{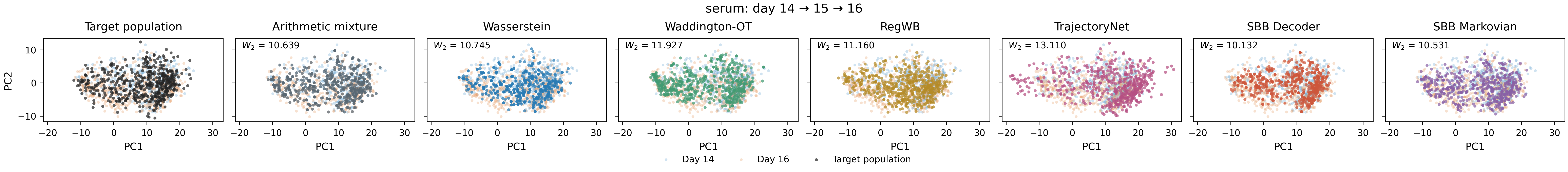}} \\
    \subfloat{
        \includegraphics[width=1\linewidth]{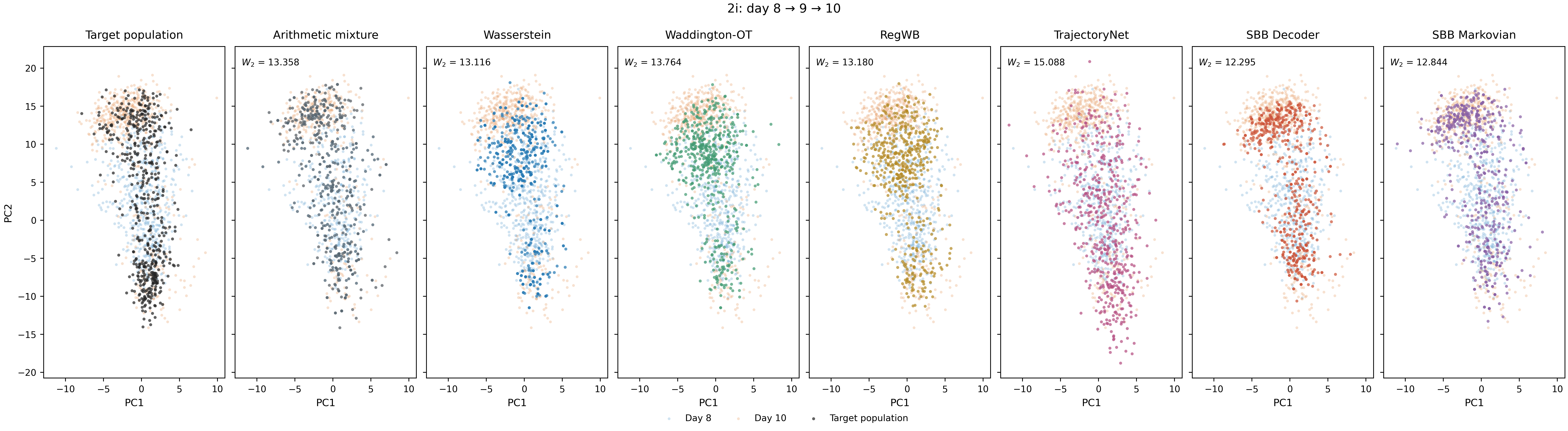}}
    \caption{Some qualitative figures for reconstruction of intermediate cell populations using 2-day intervals.}
    \label{fig:2day_qualitative_reconstruction}
\end{figure}

\clearpage
\subsubsection{3-day interval}
\begin{table}[htbp]
\centering
\small
\caption{Intermediate-state 3-day intervals reconstruction result.}
\label{tab:serum_6_2i_12_5_comparison}
\resizebox{\linewidth}{!}{%
\begin{tabular}{lcccc}
\toprule
Interpolation Method & $W_2\downarrow$ & MMD$\downarrow$ & Cell-type TV$\downarrow$ & Signature MAE$\downarrow$ \\
\midrule
\multicolumn{5}{l}{\textbf{serum condition}} \\
\midrule
Arithmetic mixture
& $12.4681 \pm 0.3015$
& $0.4484 \pm 0.0148$
& $0.1458 \pm 0.0283$
& $0.0377 \pm 0.0124$ \\
Wasserstein
& $11.8049 \pm 0.2348$
& $0.4653 \pm 0.0114$
& $0.1992 \pm 0.0676$
& $0.0354 \pm 0.0019$ \\
RegWB
& $11.9078 \pm 0.2552$
& $0.4611 \pm 0.0112$
& $0.1634 \pm 0.0511$
& \textcolor{red}{$0.0349 \pm 0.0026$} \\
TrajectoryNet
& $13.9434 \pm 0.2421$
& $0.5383 \pm 0.0125$
& $0.3945 \pm 0.0234$
& $0.0584 \pm 0.0050$ \\
Waddington-OT
& $12.1768 \pm 0.1722$
& $0.4949 \pm 0.0046$
& $0.1133 \pm 0.0263$
& $0.0422 \pm 0.0029$ \\
SBB Decoder (ours)
& \textcolor{red}{$11.3402 \pm 0.2609$}
& \textcolor{red}{$0.4213 \pm 0.0106$}
& \textcolor{red}{$0.0553 \pm 0.0316$}
& $0.0635 \pm 0.0160$ \\
SBB Markovian (ours)
& $11.8330 \pm 0.3609$
& $0.4293 \pm 0.0138$
& $0.0625 \pm 0.0220$
& $0.0384 \pm 0.0107$ \\
\midrule
\multicolumn{5}{l}{\textbf{2i condition}} \\
\midrule
Arithmetic mixture
& $11.1038 \pm 0.2972$
& $0.1605 \pm 0.0068$
& $0.1406 \pm 0.0301$
& $0.1404 \pm 0.0241$ \\
Wasserstein
& $11.3868 \pm 0.4181$
& $0.1785 \pm 0.0104$
& $0.2865 \pm 0.0479$
& $0.1596 \pm 0.0359$ \\
RegWB
& $11.5925 \pm 0.3414$
& $0.1741 \pm 0.0081$
& $0.2799 \pm 0.0352$
& $0.1556 \pm 0.0284$ \\
TrajectoryNet
& $15.0934 \pm 0.3651$
& $0.3007 \pm 0.0156$
& $0.3210 \pm 0.0177$
& $0.2251 \pm 0.0095$ \\
Waddington-OT
& $13.0232 \pm 0.4104$
& $0.2425 \pm 0.0104$
& $0.2812 \pm 0.0374$
& $0.2316 \pm 0.0437$ \\
SBB Decoder (ours)
& \textcolor{red}{$10.0232 \pm 0.4708$}
& \textcolor{red}{$0.1590 \pm 0.0106$}
& \textcolor{red}{$0.1243 \pm 0.0414$}
& $0.1417 \pm 0.0394$ \\
SBB Markovian (ours)
& $11.1622 \pm 0.4522$
& $0.1627 \pm 0.0108$
& $0.1465 \pm 0.0404$
& \textcolor{red}{$0.1228 \pm 0.0363$} \\
\bottomrule
\end{tabular}%
}
\end{table}

\begin{figure}[htp!]
    \centering
    \subfloat{
    \includegraphics[width=1\linewidth]{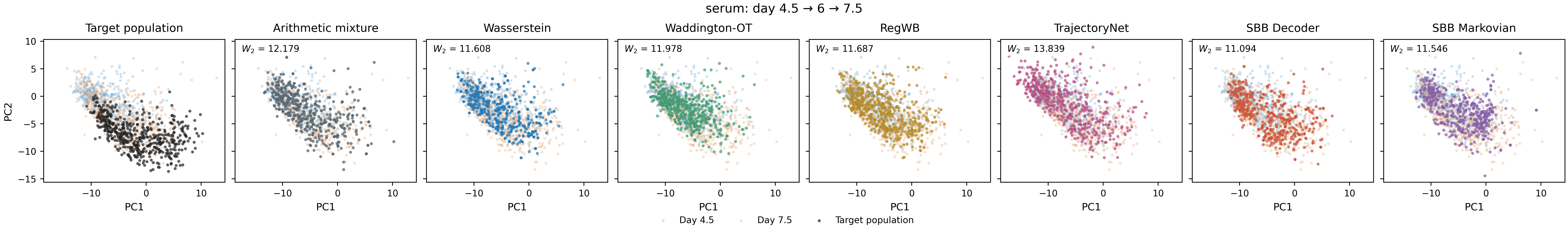}} \\
    \subfloat{
        \includegraphics[width=1\linewidth]{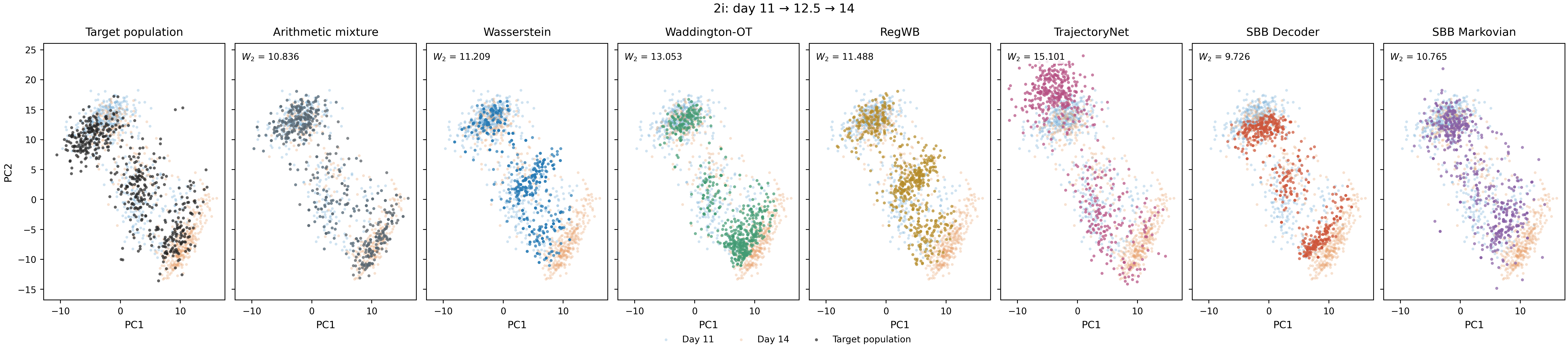}}
    \caption{Some qualitative figures for reconstruction of intermediate cell populations using 3-day intervals.}
    \label{fig:3day_qualitative_reconstruction}
\end{figure}

\clearpage

\section{RUNNING TIME}
\begin{longtable}{p{0.28\linewidth}lrrrr}
\caption{Selected training budgets and recorded times for $L$ outer iterations, $G$ directed bridges per iteration}\\
\toprule Case & Method & $L$ & $G$ & Updates & Time (s) \\
\midrule\endhead
Anisotropic Gaussian & Decoder & 50 & 3 & 298,500 & 907.6 \\
Anisotropic Gaussian & Encoder & 50 & 3 & 298,500 & 902.5 \\
Anisotropic Gaussian & Symmetric & 50 & 6 & 597,000 & 1799.7 \\
MNIST: 10 images & Decoder & 95 & 32 & 1,216,000 & 128.2 $\pm$ 6.1 \\
Ancestor, serum & Decoder & 75 & 4 & 217,800 & 176.4 $\pm$ 1.7 \\
1 day, serum & Decoder & 60 & 2 & 55,680 & 194.4 $\pm$ 2.7 \\
1 day, serum & Markovian & 60 & 2 & 55,680 & 192.9 $\pm$ 3.5 \\
1 day, 2i & Decoder & 60 & 2 & 55,680 & 215.7 $\pm$ 2.4 \\
1 day, 2i & Markovian & 60 & 2 & 55,680 & 213.4 $\pm$ 3.8 \\
2 days, 2i & Decoder & 40 & 2 & 37,120 & 162.5 $\pm$ 2.3 \\
2 days, 2i & Markovian & 40 & 2 & 37,120 & 165.3 $\pm$ 0.4 \\
2 days, serum & Decoder & 40 & 2 & 37,120 & 115.7 $\pm$ 23.1 \\
2 days, serum & Markovian & 40 & 2 & 37,120 & 115.4 $\pm$ 14.1 \\
3 days, 2i & Decoder & 40 & 2 & 37,120 & 217.0 $\pm$ 0.1 \\
3 days, 2i & Markovian & 40 & 2 & 37,120 & 200.2 $\pm$ 2.6 \\
3 days, serum & Decoder & 60 & 2 & 55,680 & 284.0 $\pm$ 2.9 \\
3 days, serum & Markovian & 60 & 2 & 55,680 & 285.9 $\pm$ 0.3 \\
\bottomrule
\end{longtable}

\end{document}